\documentclass{article} % For LaTeX2e
\usepackage{iclr2027_conference,times}

\usepackage{amsmath,amsfonts,bm}

\def\eqref#1{equation~\ref{#1}}
\def\1{\bm{1}}

\DeclareMathAlphabet{\mathsfit}{\encodingdefault}{\sfdefault}{m}{sl}
\SetMathAlphabet{\mathsfit}{bold}{\encodingdefault}{\sfdefault}{bx}{n}

\usepackage{hyperref}
\usepackage{url}

\usepackage[utf8]{inputenc} % allow utf-8 input
\usepackage[T1]{fontenc}    % use 8-bit T1 fonts
\usepackage{hyperref}       % hyperlinks
\usepackage{url}            % simple URL typesetting
\usepackage{booktabs}       % professional-quality tables
\usepackage{amsfonts}       % blackboard math symbols
\usepackage{nicefrac}       % compact symbols for 1/2, etc.
\usepackage{microtype}      % microtypography
\usepackage{xcolor}         % colors
\usepackage{amsmath}
\usepackage{bbold}
\usepackage{wrapfig}
\usepackage{booktabs}
\usepackage{graphicx}
\usepackage{multirow}
\usepackage{adjustbox}

\usepackage{amsthm}
\usepackage{mathrsfs}  
\usepackage{thmtools, thm-restate}

\newtheorem{proposition}{Proposition}
\newtheorem{remark}{Remark}

\newtheorem{assumption}{Assumption}

\title{Proper Scoring Rules for Right-Censored \\ Survival Data}

\author{Jef Jonkers\\
	IDLab \\
        Department of Electronics \\ and Information Systems \\
        Ghent University, Belgium \\
	\texttt{jef.jonkers@ugent.be}
	\And
	Glenn Van Wallendael\\
	IDLab \\
    Department of Electronics \\ and Information Systems \\
    Ghent University, Belgium \\
    imec, Belgium \\
	\And
    Luc Duchateau\\
    Biometrics Research Group \\
    Department of Morphology,\\ Imaging, Orthopedics, \\
        Rehabilitation and Nutrition \\
        Ghent University, Belgium \\
	\AND
	Sofie Van Hoecke\\
	IDLab \\
        Department of Electronics \\ and Information Systems \\
        Ghent University, Belgium \\
        imec, Belgium \\
}

\newif\ifarxiv
\arxivtrue

\ifarxiv
    \iclrfinalcopy % Uncomment for camera-ready version, but NOT for submission.
\fi
\begin{document}

\maketitle

\ifarxiv
    \fancyhead{}  % remove "Published as a conference paper at ICLR 2027"
\fi

\begin{abstract}
    Proper scoring rules provide a rigorous theoretical basis for the training and evaluation of probabilistic forecasts. In survival analysis, such forecasts describe the distribution of the time until an event occurs. However, this event time is often only partially observed because follow-up may end before the event occurs, resulting in right censoring. We propose a framework for proper scoring of right-censored survival outcomes based on a simple idea: first, map the predictive distribution through the censoring mechanism, then apply the underlying proper score on the induced observed-data law. This yields localized scores for fixed censoring times and marginalized scores when the censoring time is random or only partially observed. The resulting construction recovers familiar right-censored likelihood and IPCW-type criteria within a coherent framework, while also yielding right-censored versions of the CRPS, pinball loss, Brier score, and energy score. We show that the marginalized construction is proper under conditional independent censoring and, for strictly proper base scores, identifies the latent joint CDF on the identifiable region. The same principle also leads to censored engression, a sample-based learning objective for multivariate right-censored survival modeling. In experiments, our scores correctly rank the oracle forecast across several censoring regimes, whereas forecast-dependent plug-in weighted scores can exhibit ranking reversals. Censored engression likewise substantially improves over naive training on censored outcomes.
\end{abstract}

\section{Introduction}

Predicting time-to-event outcomes is central in many applications, including healthcare, reliability, customer churn, and risk modeling. In these settings, the goal is often not merely to estimate the mean survival time or a hazard, but to forecast the full conditional event-time distribution $F_X$ given covariates $X$. Proper scoring rules \citep{gneiting_strictly_2007} provide principled criteria for training and evaluating probabilistic forecasts, and they ensure that the true predictive distribution is optimal in expectation.

The main difficulty in survival prediction is that event times are only partially observed.  Under right censoring, one observes $Y=\min(T,C)$ and  $\Delta=\mathbb 1\{T\le C\}$, rather than the latent event time $T$ itself, where $C$ denotes the censoring time. As a result, standard proper scoring rules for uncensored outcomes cannot be applied directly. This raises a basic question: how should one define a proper scoring rule when the event time is coarsened by right censoring?

Our answer is to score forecasts after applying the same coarsening operation that generated the data. Given a predictive distribution $F$ for $T$ and a censoring time $c$, we evaluate the induced distribution of $(\min(T,c),\mathbb 1\{T\le c\})$ under $T\sim F$, rather than evaluating $F$ against the unobserved $T$. Thus, the score is determined by the observation mechanism: the censoring map is fixed independently of the forecast being evaluated. For fixed $c$, this gives a localized proper score in the sense of \citet{de_punder_localizing_2026}. For random censoring, the realized censoring time may be only partially observed; for example, if an event is observed, one only knows that $C\ge Y$. We therefore average the localized score over the conditional law of $C$ given $(Y,\Delta,X)$ under the censoring mechanism. The resulting marginalized score remains forecast-independent once the censoring law is fixed, or estimated separately from the forecast being evaluated, in contrast to plug-in constructions \citep{yanagisawa_proper_2023} where censoring weights or tail probabilities are recomputed from the forecast itself.

Beyond evaluation, the same principle naturally leads to a training objective for flexible survival models. We extend engression \citep{shen_engression_2025}, an implicit distributional regression method based on the energy score, to right-censored multivariate survival data. The resulting \emph{censored engression} method trains a noise-conditioned generator by comparing censored generated samples with observed censored outcomes, using the same localized or marginalized right-censored energy score. This makes it possible to learn flexible joint event-time distributions under censoring without requiring a tractable likelihood.

\paragraph{Contributions.}
First, building on localization by censoring the predictive distribution prior to scoring \citep{de_punder_localizing_2026}, we present a unified construction of proper scoring rules for right-censored outcomes, including settings in which the stochastic censoring time is only partially observed. Second, for random or partially observed censoring, we introduce forecast-independent marginalization over compatible censoring times, preserving propriety and, for strictly proper base scores, identifying the latent joint CDF on the identifiable region. Third, we derive concrete right-censored versions of the logarithmic score, CRPS, pinball loss, Brier score, and energy score, clarifying their relationship to existing survival criteria (see Table~\ref{tab:right-censoring-scores}). Fourth, we use the censored energy score as a training objective to define \emph{censored engression} for flexible multivariate right-censored survival modeling.

Finally, our experiments separate three empirical questions. In controlled settings with known data-generating processes (DGPs), we show that our forecast-independent scores correctly rank the oracle forecast across censoring regimes, whereas forecast-dependent plug-in weighted scores can exhibit ranking reversals. In synthetic multivariate learning experiments, we demonstrate that censored engression outperforms naive training and is competitive with likelihood-based Weibull, copula, and low-dimensional grid benchmarks trained using censored log scores. In an intensive care unit (ICU) acute kidney injury (AKI) case study, we show that censored engression can flexibly model joint time-to-event distributions for clinically relevant creatinine- and urine-output-based AKI endpoints.

\section{Related work}
Our framework is closely related to localization and coarsening. \citet{de_punder_localizing_2026} show that a proper scoring rule can be localized by applying it to a censored predictive distribution; our fixed-censoring construction is an application of this principle to right-censored survival outcomes. More broadly, \citet{gill_coarsening_1997} develop the coarsening-at-random framework for partially observed data. Our main extension begins when the censoring time is stochastic or only partially observed: we average the localized score over censoring times compatible with the observation using the censoring law, while keeping this adjustment fixed independently of the candidate forecast.

For right-censored survival evaluation, \citet{graf_assessment_1999} introduced the IPCW Brier score and its integrated version, and \citet{gerds_consistent_2006} established consistency under appropriate censoring models, including covariate-conditional censoring under conditional independence. \citet{kvamme_brier_2023} further study difficulties of standard IPCW evaluation under administrative censoring when censoring times are predictable from covariates. \citet{avati_countdown_2020} proposed the Survival-CRPS for censored outcomes, while \citet{rindt_survival_2022} later showed that this score and standard integrated Brier implementations need not be proper and established propriety of the usual right-censored log-likelihood; our construction recovers the latter and, with the appropriate conditional censoring law, the IPCW Brier score as special cases. \citet{alberge_survival_2025} propose a strictly proper censoring-adjusted separable score for survival and competing risks, with a focus on efficient stochastic optimization. Their competing-risks formulation concerns mutually exclusive event types, whereas our multivariate setting allows event times to co-occur.

Censored quantile estimation has also been studied extensively. \citet{portnoy_censored_2003} and \citet{pearce_censored_2022} develop censored quantile procedures based on reweighting censored observations using information about the event-time distribution. Among score-specific approaches, \citet{yanagisawa_proper_2023} formulates related censored pinball, Brier, and ranked probability scores and proves propriety when the required weights use the correct event-time law. Our marginalized construction instead integrates the localized score over the censoring law, which is held fixed across candidate forecasts. Thus, while fixed-censoring localization and several resulting scores are closely related to existing work, the forecast-independent marginalization over random or partially observed censoring is our main theoretical contribution.

Modern neural survival models span discrete-time competing-risk approaches such as DeepHit \citep{lee_deephit_2018}, which combines a right-censoring-adjusted likelihood for the first event and its cause with cause-specific ranking losses, and continuous-time likelihood models such as Deep Survival Machines \citep{nagpal_deep_2021}, conditional normalising flows \citep{ausset_individual_2021}, and survival mixture-density networks \citep{han_survival_2022}, which obtain flexibility through neural parametric mixtures or tractable density transformations. Adversarial approaches instead learn event-time or survival-data distributions through generative objectives, as in adversarial time-to-event modelling \citep{chapfuwa_adversarial_2018} and SurvivalGAN \citep{norcliffe_survivalgan_2023}, although the latter primarily targets synthetic survival-data generation rather than conditional survival forecasting. In contrast, censored engression learns an implicit distribution over potentially co-occurring event times by applying the censoring mechanism to generated samples and optimising a localized or marginalized proper energy score, avoiding a tractable likelihood while identifying the latent joint CDF only on the region supported by the censoring mechanism.

\section{Proper scoring under censoring}
\label{sec:cens-proper-scoring}
The development below separates an inherited ingredient from our main extension. Section~\ref{sec:fixed-cens-score} specializes the localization principle of \citet{de_punder_localizing_2026} to right-censored survival outcomes with a fixed censoring configuration. Our main theoretical extension is the marginalized construction for random or partially observed censoring times, introduced subsequently, for which Proposition~\ref{prop:marginalized-properness} gives a single propriety argument applicable to any suitable localized proper score. Deriving tractable formulas for particular scores remains score-specific.

We propose a general framework for proper scoring rules for right-censored survival outcomes that recovers, up to a forecast-independent positive factor, previously proposed proper scoring rules, such as the censored logarithmic score \citep{rindt_survival_2022} and the IPCW Brier score \citep{graf_assessment_1999, gerds_consistent_2006}. Our starting point is a strictly proper scoring rule on the latent outcome space. Following \citet{de_punder_localizing_2026}, we view censoring as a transformation of the predictive distribution itself: rather than scoring the latent law directly, we first map it to the censored outcome space and then apply the corresponding proper score to the induced observed-data law. This preserves the original score divergence after censoring and yields a score that is locally proper for the latent law on the part of the outcome space that remains identifiable under the censoring mechanism. 

Let $X\in\mathcal X\subseteq\mathbb R^d$ denote covariates, and let $T\in\mathcal T\subseteq\mathbb R_+^k$ denote the latent event-time outcome, with $k=1$ in the univariate case and $k\ge2$ in the multivariate setting. For each covariate value $x$, we seek a predictive distribution $F=\mathcal L(T\mid X=x)$. Under right censoring, however, we observe only $Y=\min(T,C)$ and $\Delta=\mathbb 1\{T\le C\}$.

Equip $\mathcal T$ with a $\sigma$-algebra $\mathscr T$, let $\mathcal P$ be a class of probability measures on $(\mathcal T,\mathscr T)$, and let $S:\mathcal P\times\mathcal T\to\bar{\mathbb R}$ be a scoring rule with divergence
\begin{equation*}
    D_S(P\|F)=\mathbb E_P[S(F,T)]-\mathbb E_P[S(P,T)].
\end{equation*}
We call $S$ strictly proper if $D_S(P\|F)\ge0$ for all $P,F\in\mathcal P$, with equality only when $P=F$. This is the full-data scoring rule that would be used if the latent event time $T$ were observed. Although it cannot generally be evaluated under censoring, it serves as the reference objective from which we construct the censored scores below.

\subsection{Localized proper scoring rules under fixed censoring}
\label{sec:fixed-cens-score}
We begin with a fixed censoring configuration. Let $A\subseteq \mathcal T$ denote the part of the outcome space that remains identifiable under this censoring mechanism, and let $*$ denote an abstract censored outcome. The corresponding censored outcome space is $\mathcal T_A^{\flat}=A\cup\{*\}$. For a latent predictive law $F\in\mathcal P$, define its censored version by
\begin{equation*}
    F_A^{\flat}
    =
    F_A+\bar F_A\,\delta_*,
    \qquad
    F_A(B):=F(B\cap A),
    \qquad
    \bar F_A:=F(A^c).
\end{equation*}
Thus, the law is left unchanged on the identifiable region $A$, while all mass outside $A$ is collapsed into the censored outcome $*$.

By a slight abuse of notation, we use $S$ also for the corresponding proper scoring rule on the censored outcome space $\mathcal T_A^\flat$. The localized score is then
\begin{equation*}
    S_A^{\flat}(F,t) := S(F_A^{\flat},t_A^{\flat}), \qquad \text{where} \quad t_A^{\flat}
    =
    \begin{cases}
    t, & t\in A,\\
    *, & t\in A^c.
    \end{cases}
\end{equation*}

Intuitively, following the localization perspective of \citet{de_punder_localizing_2026}, the forecast is first censored in the same way as the observation, and the original proper score is then evaluated on the resulting censored outcome space. By construction, the induced divergence is the original divergence applied to the censored laws, $D_{S_A^{\flat}}(P\|F) = D_S(P_A^{\flat}\|F_A^{\flat})$, so equality of population score implies equality of the induced censored distributions, not necessarily of the full latent laws. Thus, strictness applies to the induced censored law; recovery of the latent law is limited to what the censoring mechanism identifies.

For fixed right censoring at time $c$, the identifiable region depends on the boundary convention. With the convention $\Delta=\mathbb 1\{T\le c\}$, the identifiable region is $A_c=[0,c]$, and the censored outcome map is
\begin{equation*}
     \psi^\flat_c(t)
    =
    \begin{cases}
    t, & t\le c,\\
    *, & t>c.
    \end{cases}
\end{equation*}
The corresponding censored law is $F_c^{\flat} = F(\cdot\cap[0,c]) + F((c,\infty))\delta_*$. Thus, an event at the boundary $c$ is distinguished from a censored observation beyond $c$. We write $S_c^\flat(F;Y,\Delta)$ for the localized score induced by this fixed censoring time.

When $T\mid X$ is continuous, boundary events have probability zero, so the convention at $T=c$ does not affect expected scores.  For distance-based scores such as the CRPS and energy score, however, the abstract censored state $*$ cannot be inserted directly into Euclidean distances. By a slight abuse of notation, we keep the notation $\psi_c^\flat$ for this numerical encoding of the same coarsened observation:
\begin{equation*}
    \psi^\flat_c(t)=\min(t,c),
    \qquad
    F_c^\flat=\mathcal L(\psi^\flat_c(T)).
\end{equation*}
This representation should be read as an encoding of the coarsened observation: the value $c$ represents the censored region $(c,\infty)$, not an event known to occur at $c$. When the event indicator is needed, we keep the convention $\Delta=\mathbb 1\{T\le c\}$, so that $(c,0)$ denotes censoring beyond $c$ whereas $(c,1)$ denotes an event observed at the boundary.

The single-atom representation above describes scalar censoring. In the multivariate setting, componentwise censoring retains which coordinates are censored, and we therefore work directly with the componentwise censoring map,
\begin{equation*}
    \psi^\flat_c(t) = (t_1\wedge c_1,\dots,t_k\wedge c_k), \qquad F_c^\flat=\mathcal L(\psi^\flat_c(T)),
\end{equation*}
where $c=(c_1,\dots,c_k)$. A shared scalar censoring time corresponds to the special case $c_1=\cdots=c_k=c$.

\subsection{Marginalized proper scoring rules}
We now turn from a fixed censoring threshold to random right censoring. Suppose that $Y=\min(T,C)$, $\Delta=\mathbb 1\{T\le C\}$, and let $H(t\mid x)=\mathbb P(C\le t\mid X=x)$ denote the conditional censoring distribution. To relate scoring of the observed data $(Y,\Delta)$ to the latent conditional event-time distribution $T\mid X$, we assume conditional independent censoring.

\begin{assumption}[Conditional independent censoring]
    \label{ass:indp-cens}
    $T\perp C\mid X.$
\end{assumption}

Under Assumption~\ref{ass:indp-cens}, the censoring mechanism carries no information about the latent event time beyond the observed covariates. Importantly, this assumption is distinct from whether the realized censoring time is observed. In the usual right-censoring setting, $C$ may itself be only partially observed: if $\Delta=1$, for example, the event is observed at $Y=T$ and one only learns that $C\ge Y$.

Let $\Pi(dc\mid Y,\Delta,X) := \mathcal L(C\in dc\mid Y,\Delta,X)$ denote the conditional law of the censoring time compatible with the observed data, induced by $H(\cdot\mid X)$ and the right-censoring mechanism. We define the observed-data score by averaging the fixed-$c$ localized score over this conditional law:
\begin{equation*}
    \bar S(F;Y,\Delta,X) := \int S_c^\flat(F;Y,\Delta)\, \Pi(dc\mid Y,\Delta,X).
\end{equation*}
Thus, the score first conditions on each censoring configuration compatible with the observation and then averages over those configurations according to the censoring law, which is fixed independently of the forecast $F$.

In the univariate case, this reduces to
\begin{equation*}
    \bar S(F;Y,\Delta,X) =
    \begin{cases}
        \displaystyle
        \int_{[Y,\infty)} S_c^\flat(F;Y,1)\,H(dc\mid C\ge Y,X),     & \Delta=1,\\[2mm]
        S_Y^\flat(F;Y,0),   & \Delta=0.
    \end{cases}
\end{equation*}

When $\Delta=0$, the censoring time is observed exactly as $C=Y$; when $\Delta=1$, it is only known to satisfy $C\ge Y$. Similarly, with a shared censoring time in the multivariate case, $C$ is observed whenever at least one coordinate is censored. If all coordinates are observed, one only knows that $C\ge \max_j Y_j$, and $\Pi$ is the law of $C$ conditional on this event.

The censoring mechanism also determines which parts of the latent distribution are identifiable. For fixed $x$, let
\begin{equation*}
    \bar{\mathcal S}_x(F) := \mathbb E_P[
        \bar S(F;Y,\Delta,X)
        \mid X=x
    ]
\end{equation*}
denote the conditional expected score. For a shared scalar censoring time, define
\begin{equation*}
    \mathcal I_x^{\mathrm{shared}}
    =
    \left\{
    t\in\mathbb R_+^k:
    \mathbb P\!\left(
        C> \max_{j=1,\ldots,k}t_j
        \mid X=x
    \right)>0
    \right\},
\end{equation*}
whereas for componentwise censoring times
$C=(C_1,\ldots,C_k)$, define
\begin{equation*}
    \mathcal I_x^{\mathrm{comp}}
    =
    \left\{
    t\in\mathbb R_+^k:
    \mathbb P\!\left(
        C_j>t_j
        \text{ for all }j=1,\ldots,k
        \mid X=x
    \right)>0
    \right\}.
\end{equation*}

We write $\mathcal I_x$ for the relevant identifiable region under the censoring mechanism considered. In the univariate case, both definitions reduce to $\mathcal I_x=\{t\ge0:\mathbb P(C>t\mid X=x)>0\}$.

The following proposition shows that averaging over the compatible censoring configurations preserves propriety, while strict propriety can hold only on the region that is identifiable under the censoring mechanism.

\begin{proposition}[Marginalized propriety under right censoring]
    \label{prop:marginalized-properness}
    Assume Assumption \ref{ass:indp-cens}, and let $H(\cdot\mid X)$ denote the true conditional distribution of the censoring time. Suppose that the conditional law of $C$ given $(Y,\Delta,X)$ used in the marginalized score is the one induced by $H(\cdot\mid X)$ and the right-censoring mechanism. If, for $H(\cdot\mid X=x)$-almost every censoring time $c$, the localized score $S_c^\flat$ is proper for the induced censored law, then, for every $x$, $\bar{\mathcal S}_x(F) \ge \bar{\mathcal S}_x(P)$ for all forecasts $F$.
    
    If, in addition, the localized scores are strictly proper, then
    \begin{equation*}
        \bar{\mathcal S}_x(F) = \bar{\mathcal S}_x(P) 
        \quad
        \Longleftrightarrow\quad F_{c,x}^\flat=P_{c,x}^\flat 
        \quad
        \text{for }H(\cdot\mid x)\text{-almost every }c.
    \end{equation*}
    
    Moreover, equality of the marginalized scores implies
    \begin{equation*}
        F_x(t)=P_x(t) \qquad \text{for every }t\in\mathcal I_x,
    \end{equation*}
    where $F_x$ and $P_x$ denote the joint conditional CDFs.
    
    In particular, if $\mathcal I_x=\mathbb R_+^k$, then $F_x=P_x$ as probability laws.
\end{proposition}

Outside $\mathcal I_x$, two latent distributions may differ while inducing the same observed censored law, so no scoring rule based only on $(Y,\Delta,X)$ can be strictly proper for such non-identifiable differences. This matters whenever a downstream quantity depends on the distribution beyond supported follow-up: for example, under administrative follow-up to two years, a five-year event probability necessarily relies on extrapolation. The same applies to upper quantiles, mean survival time, and other tail-dependent functionals.

The conclusion of Proposition~\ref{prop:marginalized-properness} also relies on Assumption~\ref{ass:indp-cens}. If conditional independent censoring fails, the marginalized score remains a well-defined observed-data criterion under the specified censoring mechanism, but its population minimizer need not equal $\mathcal L(T\mid X)$. For example, an unobserved factor may affect both $T$ and $C$. Without additional assumptions or a joint model for $(T,C)$, the resulting distortion is generally not identifiable from $(Y,\Delta,X)$ alone.

In the univariate continuous case, taking the logarithmic score as the base score recovers the usual right-censored negative log-likelihood; see Section~\ref{app:log-score-censoring} and Table~\ref{tab:right-censoring-scores}. In this case, marginalization simplifies because the fixed-$c$ localized contribution is either observed exactly or is constant over the censoring times compatible with the observation.

\begin{table*}[t]
\centering
\small
\setlength{\tabcolsep}{4pt}
\renewcommand{\arraystretch}{1.2}
\caption{Examples of uncensored, fixed-censoring localized, and marginalized scores under right censoring. All scores are written as losses, so lower values are better. We write
$G(t\mid x)=\mathbb P(C>t\mid X=x)$ and $G(t^-\mid x)=\mathbb P(C\ge t\mid X=x)$.}
\label{tab:right-censoring-scores}
\begin{adjustbox}{max width=\textwidth}
\begin{tabular}{p{2.3cm} p{11.8cm}}
\toprule
Type & Score \\
\midrule

\multicolumn{2}{l}{\textbf{Logarithmic score (log)}}\\
Uncensored &
$S_{\log}(F;t,x)=-\log f(t\mid x)$ \\[1mm]
Fixed local &
$S_{\log,c}^{\flat}(F;y,\delta,x)=-\delta\log f(y\mid x)-(1-\delta)\log (1 - F(c\mid x))$ \\[1mm]
Marginalized &
$\bar S_{\log}(F;y,\delta,x)=-\delta\log f(y\mid x)-(1-\delta)\log (1-F(y\mid x))$ \\
\addlinespace[1mm]

\multicolumn{2}{l}{\textbf{Continuous ranked probability score (CRPS)}}\\
Uncensored &
$S_{\mathrm{CRPS}}(F;t,x)=\int_0^t F(s\mid x)^2\,ds+\int_t^\infty (1-F(s\mid x))^2\,ds$ \\[1mm]
Fixed local &
$S_{\mathrm{CRPS},c}^{\flat}(F;y,\delta,x)=\int_0^y F(s\mid x)^2\,ds+\delta\int_y^c (1-F(s\mid x))^2\,ds$ \\[1mm]
Marginalized &
$\bar S_{\mathrm{CRPS}}(F;y,\delta,x)=\int_0^y F(s\mid x)^2\,ds+\delta\int_y^\infty \frac{G(s\mid x)}{G(y^-\mid x)}(1-F(s\mid x))^2\,ds$ \\
\addlinespace[1mm]

\multicolumn{2}{l}{\textbf{Brier score at horizon $\tau$ ($\mathbf{\text{BS}_\tau}$)}}\\
Uncensored &
$S_{\mathrm{BS},\tau}(F;t,x)=\bigl(F(\tau\mid x)-\mathbb 1\{t\le \tau\}\bigr)^2$ \\[1mm]
Fixed local &
$S_{\mathrm{BS},\tau,c}^{\flat}(F;t,x)=\bigl(F_c^{\flat}(\tau\mid x)-\mathbb 1\{\min(t,c)\le \tau\}\bigr)^2$ \\[1mm]
Marginalized &
$\bar S_{\mathrm{BS},\tau}(F;Y,\Delta,X)=\mathbb 1\{Y>\tau\}F(\tau\mid X)^2+\Delta\mathbb 1\{Y\le \tau\}\frac{G(\tau\mid X)}{G(Y^-\mid X)}(1-F(\tau\mid X))^2$ \\
\addlinespace[1mm]

\multicolumn{2}{l}{\textbf{Pinball / quantile score at level $\alpha$ ($\mathbf{\text{Q}_\alpha}$)}}\\
Uncensored &
$S_{\mathrm{Q},\alpha}(F;t,x)=\bigl(\alpha-\mathbb 1\{t<q_\alpha(F\mid x)\}\bigr)\bigl(t-q_\alpha(F\mid x)\bigr)$ \\[1mm]
Fixed local &
$S_{\mathrm{Q},\alpha,c}^{\flat}(F;t,x)=\bigl(\alpha-\mathbb 1\{\min(t,c)<q_\alpha(F_c^{\flat}\mid x)\}\bigr)\bigl(\min(t,c)-q_\alpha(F_c^{\flat}\mid x)\bigr)$ \\[1mm]
Marginalized &
\(\begin{aligned}
\bar S_{\mathrm{Q},\alpha}(F;Y,\Delta,X)
&= \alpha(Y-q_\alpha(F\mid X))_+ \\
&\quad + \Delta\frac{1-\alpha}{G(Y^-\mid X)}
\mathbb 1\{q_\alpha(F\mid X)>Y\}
\int_Y^{q_\alpha(F\mid X)} G(t\mid X)\,dt
\end{aligned}\) \\
\addlinespace[1mm]

\multicolumn{2}{l}{\textbf{Energy score (ES)}}\\
Uncensored &
$S_{\mathrm{ES},\beta}(F;t,x)=\mathbb E[\|Z-t\|^\beta\mid X=x]-\frac12\mathbb E[\|Z-Z'\|^\beta\mid X=x]$ \\[1mm]
Fixed local &
$S_{\mathrm{ES},\beta,c}^{\flat}(F;y,x)=\mathbb E[\|\psi_c(Z)-y\|^\beta\mid X=x]-\frac12\mathbb E[\|\psi_c(Z)-\psi_c(Z')\|^\beta\mid X=x]$ \\[1mm]
Marginalized &
$\bar S_{\mathrm{ES},\beta}(F;Y,\Delta,X)
=
\mathbb E\!\left[
S_{\mathrm{ES},\beta,C}^{\flat}(F;Y,X)
\mid Y,\Delta,X
\right]$  \\
\bottomrule
\end{tabular}
\end{adjustbox}
\end{table*}

Table~\ref{tab:right-censoring-scores} summarizes several scores obtained from the localized and marginalized constructions. The formulas use the continuous no-boundary-ambiguity convention, under which the abstract censored outcome can be represented by the boundary value $c$ through $\psi_c^\flat(t)=\min(t,c)$. Detailed derivations are given in Appendix~\ref{app:score-derivations}.

\begin{remark}[Fixed, administrative, and observed censoring times]
    If $C=c$ is fixed by design, or $C=c(X)$ is deterministic given the observed covariates, the localized score can be evaluated directly and no marginalization is required; conditional independent censoring then holds automatically.
    
    If $C$ is stochastic but observed for every individual, marginalization is again unnecessary: the localized score can be evaluated at the observed $C$. Observing $C$, however, does not imply Assumption~\ref{ass:indp-cens}; latent factors may affect both $T$ and $C$, so that $T\not\!\perp C\mid X$. Under conditional independent censoring, the expected localized score is proper on the identifiable region. This setting is relevant, for example, when discharge or death times are recorded in the ICU application of Section~\ref{sec:aki-use-case}.
\end{remark}

\paragraph{Estimating the censoring law.}
The construction assumes that the censoring law $H(\cdot\mid X)$ is known or estimated independently of the forecast being evaluated. When it is not known from the study design, it can be estimated on a separate calibration sample or by cross-fitting the evaluation data. In the latter case, the censoring model is fitted on $K-1$ folds and used to score observations in the held-out fold. The same fitted censoring models and folds are used for all candidate forecasts, which remain fixed throughout; in particular, the censoring adjustment is never recomputed from the forecast being evaluated.

\section{Censored engression for multivariate survival}
\label{sec:censored-engression}
Thus far, we have used censoring-aware proper scoring rules to \emph{evaluate} predictive distributions under right censoring. The same scores can also serve as learning objectives. We develop this connection for the energy score by combining two ideas: the implicit distributional representation of engression \citep{shen_engression_2025} and the localization-by-censoring principle of Section~\ref{sec:fixed-cens-score}. The resulting \emph{censored engression} method learns flexible, possibly multivariate, event-time distributions by applying the censoring map to generated samples before evaluating the energy score, with marginalization over compatible censoring configurations when necessary.

\paragraph{Background on engression.} Engression \citep{shen_engression_2025} represents the conditional law through a generator
\begin{equation}
    \label{eq:noise-model}
    Y_g(x)=g(x,\varepsilon),
    \qquad
    \varepsilon\sim P_\varepsilon,
\end{equation}
where $P_\varepsilon$ is a simple reference distribution. The generator induces the conditional predictive law $P_g(\cdot\mid x):=\mathcal L(g(x,\varepsilon))$, so unlike classical regression, engression directly models the full conditional distribution via samples. For two independent noise variables $\varepsilon,\varepsilon'$, the energy score population objective is
\begin{equation*}
    \mathcal R_{\mathrm{ES}}(g)
    =
    \mathbb E\left[
        \|Y-g(X,\varepsilon)\|
        -
        \frac12\|g(X,\varepsilon)-g(X,\varepsilon')\|
    \right],
\end{equation*}
where $(X,Y)\sim P$. Under correct specification, minimizing this objective recovers the true conditional law; see Proposition~1 of \citet{shen_engression_2025}.

\subsection{Censored engression}
We now extend engression to multivariate right-censored survival outcomes. Let $T=(T_1,\ldots,T_k)\in\mathbb R_+^k, k \ge 1$, denote the latent event-time vector, and let the generator in \eqref{eq:noise-model} induce the latent predictive law $F_{g,x}:=\mathcal L(g(x,\varepsilon)).$

Under right censoring, the event times are not fully observed. For a fixed censoring time $c$, let $\psi^\flat_c(t)=\bigl(\min(t_1,c),\ldots,\min(t_k,c)\bigr)$ denote the shared censoring map, and let $Y=\psi^\flat_c(T)$ denote the corresponding censored outcome. Following the localization-by-censoring principle, we apply the energy score after censoring, that is, to the induced censored law $\mathcal L(\psi^\flat_c(g(x,\varepsilon)))$ rather than to the latent law $F_{g,x}$ itself. The resulting fixed-$c$ censored energy loss is
\begin{equation}
	\label{eq:local-cens-engression}
    \mathcal R_{\mathrm{local\text{-}cens\text{-}ES}}(g;Y,x,c)
    =
    \mathbb E_{\varepsilon}
    \left[
        \|\psi_c(g(x,\varepsilon))-Y\|
    \right]
    -
    \frac12
    \mathbb E_{\varepsilon,\varepsilon'}
    \left[
        \|\psi_c(g(x,\varepsilon))
        -
        \psi_c(g(x,\varepsilon'))\|
    \right],
\end{equation}
where $\varepsilon,\varepsilon'\stackrel{\mathrm{iid}}{\sim}P_\varepsilon$.

When the censoring time is not fully observed, we average the fixed-$c$ loss over the compatible conditional law of $C$ given the observed data and then over the observed-data distribution:
\begin{equation}
	\label{eq:cens-engression}
    \mathcal R_{\mathrm{cens\text{-}ES}}(g) :=
    \mathbb E\!\left[
        \mathbb E\!\left[
            \mathcal R_{\mathrm{local\text{-}cens\text{-}ES}}
            (g;Y,X,C)
            \mid Y,\Delta,X
        \right]
    \right].
\end{equation}

The inner conditional expectation is exactly the observation-level marginalized right-censored energy score introduced above, while $\mathcal R_{\mathrm{cens\text{-}ES}}(g)$ is its population expectation. When there is no censoring, $\psi_c$ is the identity map and the loss reduces to the ordinary engression energy score. With censoring, the model is instead trained to match the law of the censored outcome, and therefore recovers the latent event-time law only to the extent that it remains identifiable under the censoring mechanism. We refer to this procedure as \emph{censored engression}. We also present this approach with a population optimality guarantee; see Proposition \ref{prop:censored-engression}.

\section{Experiments}
\label{sec:experiments}
\subsection{Proper scoring rules for evaluating forecasts}
\label{sec:forecast-evaluation}
\begin{wraptable}{r}{0.45\textwidth}
\vspace{-1.0em}
\centering
\scriptsize
\setlength{\tabcolsep}{4.0pt}
\renewcommand{\arraystretch}{0.92}
\caption{
Rank of the oracle forecast. Rank \(1\) means best mean score. 
Regimes A--C compare \(F_0\) against \(F_1,\ldots,F_4\); regime D compares \(F^{(D)}_0\) against four exploit forecasts \(F_5(\epsilon)\).
}
\label{tab:oracle-ranks}
\begin{tabular}{lcccc}
\toprule
Scoring rule & A & B & C & D \\
\midrule
CRPS, latent & 1 & 1 & 1 & 1 \\
Pinball, latent & 1 & 1 & 1 & 1 \\
Brier, latent & 1 & 1 & 1 & 1 \\
NLL, latent & 1 & 1 & 1 & 1 \\
\midrule
NLL~\citep{yanagisawa_proper_2023}, plug-in & 2 & 1 & 2 & 2 \\
Brier MC~\citep{yanagisawa_proper_2023}, plug-in & 1 & 1 & 1 & 5 \\
Brier~\citep{yanagisawa_proper_2023}, plug-in & 1 & 1 & 2 & 5 \\
Pinball~\citep{yanagisawa_proper_2023}, plug-in & 1 & 1 & 1 & 1 \\
RPS~\citep{yanagisawa_proper_2023}, plug-in & 2 & 1 & 2 & 5 \\
\midrule
Censored NLL, ours & 1 & 1 & 1 & 1 \\
CRPS, ours & 1 & 1 & 1 & 1 \\
Pinball, ours & 1 & 1 & 1 & 1 \\
Brier, ours & 1 & 1 & 1 & 1 \\
\bottomrule
\end{tabular}
\vspace{-1.0em}
\end{wraptable}
We use controlled simulations with known data-generating laws to test whether the proposed scores preserve the oracle ranking under censoring. We compare our forecast-independent scores with scores of \citet{yanagisawa_proper_2023} using either oracle weights or plug-in weights recomputed from each candidate forecast; full simulation details are given in Appendix \ref{app:forecast-evaluation-setup}.

\paragraph{Results and discussion.}
Table~\ref{tab:oracle-ranks} shows that, across all four censoring regimes, the latent gold-standard scores and all of our forecast-independent scores rank the oracle first. The oracle-weight scores of \citet{yanagisawa_proper_2023} likewise rank the oracle first, whereas forecast-dependent plug-in versions can reverse the ranking: the plug-in NLL ranks the oracle second in regimes A, C, and D, the plug-in RPS ranks it second in A and C and last in D, and plug-in Brier variants also fail in C or D. Corresponding raw score values are reported in Appendix \ref{app:forecast-evaluation-score-values}. Sensitivity analyses in Appendix \ref{app:censoring-estimation-sensitivity} show that the ranking is largely stable when the censoring law is estimated rather than known, with only mild degradation under misspecification.

\subsection{Synthetic multivariate censored-engression experiment}
\label{sec:synthetic-censored-engression}
We next study censored engression on multivariate event times $T=(T_1,\ldots,T_k)$, $k\in\{2,3,5,10\}$, under a shared censoring time. The main DGP is a covariate-dependent mixture log-normal distribution, yielding multimodal conditional laws outside the Weibull-copula families used by the likelihood baselines. We compare censored engression with naive observed engression, neural Weibull and copula models trained by censored likelihood, and low-dimensional Gumbel-copula and grid baselines; full details are in Appendix \ref{app:synthetic-censored-engression}. Common competing-risk models such as DeepHit and Deep Survival Machines do not directly model our target because event times may co-occur; fitting them separately for each endpoint would discard their joint dependence.

\paragraph{Results.}
Censored energy score is our primary observed-data criterion, with latent energy score serving as a synthetic diagnostic. Censored engression obtains the lowest censored energy score among feasible methods across the reported censoring regimes and dimensions. Under uniform censoring, for example, it improves over naive observed engression from $0.5524$ to $0.4904$ at $k=2$ and from $1.6402$ to $1.3382$ at $k=10$. At $k=2$, the strongest likelihood baseline, MLP Clayton--Weibull, obtains $0.5152$. Likelihood-based models can nevertheless achieve lower latent energy score, reflecting behavior beyond the region identified by censored observations. Full results and the unimodal sanity check are reported in Appendix \ref{app:synthetic-censored-engression-results}.

\subsection{Illustrative clinical use case: early AKI forecasting in the ICU}
\label{sec:aki-use-case}

We use early acute kidney injury (AKI) forecasting in the intensive care unit (ICU) as a realistic illustration of censored engression on multivariate time-to-event data. At each prediction time, the goal is to forecast the joint distribution of the remaining time to two clinically meaningful endpoints: creatinine-defined AKI and urine-output-defined AKI. The predictors are routinely available variables, including demographics, baseline renal function, elapsed ICU time, and summary features of recent creatinine and urine-output measurements over the preceding $72$ hours. The analysis uses de-identified ICU data from MIMIC-IV \citep{johnson_mimic-iv_2023}; additional cohort, endpoint, and evaluation details are given in Appendix~\ref{app:aki-additional-results}. % TODO: add the appropriate MIMIC-IV citation key in the final bibliography.

The observed target is right-censored by ICU discharge or death. This exit time is observed in the database, so the localized version of our scoring framework can be used directly: for each patient and prediction time, generated event times and observed outcomes are both mapped through the same censoring time before the score is evaluated. This removes the need to marginalize over unknown censoring times, but does not imply conditional independent censoring. We assume that the latent AKI times and ICU exit time are conditionally independent given the available predictors; if residual dependence remains, for example through unmeasured clinical severity, the results should be interpreted as predictive performance under the observed exit mechanism rather than as identification of latent post-exit AKI incidence.

We report localized energy scores for the joint outcome and each marginal endpoint, together with localized Brier scores for ``either AKI'' at clinically relevant horizons. We focus on proper distributional scores rather than the C-index, which evaluates discrimination after reducing the multivariate prediction to an endpoint or scalar risk score. Table~\ref{tab:censored-engression-selected-results} compares censored engression with a naive observed-engression baseline using the same generator architecture.

\begin{table}[t]
\centering
\small
\setlength{\tabcolsep}{4.0pt}
\caption{
Selected results for censored engression. 
Lower scores are better. 
Synthetic rows report both observed censored energy score and latent energy score under uniform censoring for the mixture log-normal DGP. 
DGP is the true data-generating law and latent engression is an infeasible reference trained on uncensored event times. 
AKI rows report localized energy scores; no latent oracle is available in the real-data use case. 
Bold indicates the best feasible result among the displayed methods.
}
\label{tab:censored-engression-selected-results}

\begin{adjustbox}{max width=\textwidth}
\begin{tabular}{llcccc}
\toprule
Setting & Score & DGP & Latent engression & Naive obs. engression & Censored engression \\
\midrule
Synthetic, \(k=2\), uniform 
& Censored ES 
& \(0.4840\) & \(0.4907\) & \(0.5524\) & \(\mathbf{0.4904}\) \\
& Latent ES 
& \(1.2468\) & \(1.2601\) & \(1.4440\) & \(\mathbf{1.2887}\) \\
\addlinespace[0.2em]

Synthetic, \(k=10\), uniform 
& Censored ES 
& \(1.3131\) & \(1.6082\) & \(1.6402\) & \(\mathbf{1.3382}\) \\
& Latent ES 
& \(9.4613\) & \(10.0111\) & \(16.5122\) & \(\mathbf{14.2576}\) \\
\midrule

AKI, joint 
& Loc. ES 
& -- & -- & \(43.442\) & \(\mathbf{30.933}\) \\
AKI, creatinine 
& Loc. ES 
& -- & -- & \(23.591\) & \(\mathbf{15.439}\) \\
AKI, urine-output 
& Loc. ES 
& -- & -- & \(32.981\) & \(\mathbf{23.769}\) \\
\bottomrule
\end{tabular}
\end{adjustbox}

\end{table}

Censored engression improves all localized energy scores in Table~\ref{tab:censored-engression-selected-results}: the joint score decreases from 43.442 to 30.933, the creatinine-AKI score from 23.591 to 15.439, and the urine-output-AKI score from 32.981 to 23.769. Localized Brier scores for ``either AKI'' likewise decrease from 0.023 to 0.022 at 6 hours, 0.045 to 0.040 at 12 hours, and 0.081 to 0.071 at 24 hours.

Appendix~\ref{app:aki-additional-results} compares censored engression with standard survival baselines. It obtains the lowest localized Brier scores across the reported horizons (Appendix Table~\ref{tab:aki-utility-compact}) and the lowest joint and marginal localized energy scores among methods that generate full event-time distributions (Appendix Table~\ref{tab:aki-es-baselines}).

\section{Conclusion}
This work argues that censored probabilistic evaluation should be built on forecast-independent censoring adjustment. IPCW scores target the latent event-time distribution only when the weights are determined by the censoring mechanism rather than by the forecast being evaluated. This preserves the oracle ranking across a range of censoring regimes, while forecast-dependent plug-in scores change the evaluation target and can induce ranking reversals.

The same methodological principle yields a training objective for censored generative survival modeling. Censored engression uses the censored energy score to extend engression to right-censored multivariate outcomes, improving over naive training on censored observations.

An important extension is to more general coarsening mechanisms, including interval censoring. Recent work by \citet{yanagisawa_strictly_2026} develops a strictly proper interval-censored log score under additional censoring assumptions; extending the localization-and-marginalization perspective developed here to broader classes of proper scores under interval and set censoring remains an interesting direction.

\paragraph{Limitations.}
Our guarantees rely on conditional independent censoring and a correctly specified censoring law. If $T\not\!\perp C\mid X$, the score need not target the marginal latent event-time law. Strict propriety is also limited to the identifiable region, so long-horizon risks and other tail-dependent quantities may require extrapolation. In practice, the censoring law should be estimated separately and then held fixed across candidate forecasts; misspecification can still affect ranking. %Future work should study uncertainty propagation for the censoring model and joint learning of event-time and censoring mechanisms.

% \subsection*{Ethics statement}

% (This section is \textbf{recommended} and does not count toward the page limit.)

% If authors feel that their paper submission raises questions regarding the Code
% of Ethics, they are encouraged to include a paragraph of Ethics Statement (at
% the end of the main text before references) to address potential concerns where
% appropriate. Topics include, but are not limited to, studies that involve human
% subjects, practices to data set releases, potentially harmful insights,
% methodologies and applications, potential conflicts of interest and sponsorship,
% discrimination/bias/fairness concerns, privacy and security issues, legal
% compliance, and research integrity issues (e.g., IRB, documentation, research
% ethics). This statement should not be more than 1 page.

\ifarxiv
\subsection*{Reproducibility statement}
The proper scoring rules presented in this work are implemented in Python package \texttt{probseer}, see \url{https://github.com/predict-idlab/probseer}.
\else
\subsection*{AI use statement}

In this work, we used generative AI tools for language-related assistance, including grammar correction, wording refinement, and text styling. We also used generative AI tools to assist with code generation, debugging, and general programming-related questions, as well as to support parts of the literature research process, such as identifying potentially relevant papers and research directions. We have not used generative AI tools for any other tasks requiring disclosure, and other disclosure categories are not applicable to this work.

We reviewed and verified all AI-assisted work. AI-assisted text was manually reviewed and edited by the authors. AI-generated or AI-assisted code was inspected, tested, and verified for correctness before being used. Literature identified with the assistance of generative AI was independently verified by consulting the original publications and relevant sources; claims and citations were not included solely on the basis of AI-generated output.

We take responsibility for the final content of this work, including all text, claims, code, and other artifacts produced with the aid of generative AI.

\subsection*{Reproducibility statement}
The code accompanying this work will be made publicly available upon publication.
\fi


\ifarxiv
\subsubsection*{Acknowledgments}
Jef Jonkers is funded by the Research Foundation Flanders (FWO, Ref. 1S11525N).
\fi

\bibliography{references}
\bibliographystyle{iclr2027_conference}

\appendix
\section{Theory and proofs}
\label{app:proofs}

\subsection{Proof of Proposition~\ref{prop:marginalized-properness}}

We use $S_c^\flat(F;\psi_c^\flat(t))$ and $S_c^\flat(F;Y,\Delta)$ interchangeably for the abstract and right-censoring encodings of the same fixed-$c$ censored observation.

\begin{proof}
	Fix $x$. Let $P_x$ denote the true conditional law of $T\mid X=x$ and let $F_x$ denote a candidate forecast. Let $H(\cdot\mid x)$ be the true conditional distribution of the censoring time $C\mid X=x$. For each fixed censoring time $c$, define the induced censored laws
	\begin{equation*}
		P_{c,x}^{\flat}:=\mathcal L(\psi_c^\flat(T)\mid X=x),
		\qquad
		F_{c,x}^{\flat}:=\mathcal L(\psi_c^\flat(Z)\mid X=x),
		\quad Z\sim F_x.
	\end{equation*}
	
	By definition,
	\begin{equation*}
		\bar{\mathcal S}_x(F)-\bar{\mathcal S}_x(P)
		=
		\mathbb E_P\!\left[
		\bar S(F;Y,\Delta,X)-\bar S(P;Y,\Delta,X)
		\mid X=x
		\right].
	\end{equation*}
	The marginalized score averages the fixed-$c$ localized score over the conditional law of $C$ given the observed data,
	\begin{equation*}
		\Pi(dc\mid Y,\Delta,X)
		:=
		\mathcal L(C\in dc\mid Y,\Delta,X),
	\end{equation*}
	where this conditional law is induced by the true censoring distribution $H(\cdot\mid X)$. Hence, by the tower property,
	\begin{align*}
		\bar{\mathcal S}_x(F)-\bar{\mathcal S}_x(P)
		&=
		\mathbb E_P\!\left[
		\mathbb E_P\!\left[
		S_C^{\flat}(F_x;Y,\Delta)
		-
		S_C^{\flat}(P_x;Y,\Delta)
		\mid Y,\Delta,X
		\right]
		\middle| X=x
		\right] \\
		&=
		\mathbb E_P\!\left[
		S_C^{\flat}(F_x;Y,\Delta)
		-
		S_C^{\flat}(P_x;Y,\Delta)
		\mid X=x
		\right].
	\end{align*}
	
	Under conditional independent censoring, $T\perp C\mid X$, we may condition on the realized censoring time and integrate over its true conditional law:
	\begin{align*}
		\bar{\mathcal S}_x(F)-\bar{\mathcal S}_x(P)
		&=
		\int
		\mathbb E_{P_x}\!\left[
		S_c^{\flat}(F_x;\psi_c^\flat(T))
		-
		S_c^{\flat}(P_x;\psi_c^\flat(T))
		\right]
		H(dc\mid x).
	\end{align*}
	For fixed $c$, the distribution of $\psi_c^\flat(T)$ under the truth is exactly $P_{c,x}^{\flat}$. Therefore the inner expectation is the divergence of the localized score between the induced censored laws:
	\begin{equation*}
		\mathbb E_{P_x}\!\left[
		S_c^{\flat}(F_x;\psi_c^\flat(T))
		-
		S_c^{\flat}(P_x;\psi_c^\flat(T))
		\right]
		=
		D_{S_c^{\flat}}\!\left(P_{c,x}^{\flat}\|F_{c,x}^{\flat}\right).
	\end{equation*}

    Hence
    \begin{equation*}
        \bar{\mathcal S}_x(F)-\bar{\mathcal S}_x(P)
        =
        \int
        D_{S_c^{\flat}}\!\left(
            P_{c,x}^{\flat}\|F_{c,x}^{\flat}
        \right)
        H(dc\mid x).
    \end{equation*}
    
    By propriety of the fixed-censoring localized score, the integrand is nonnegative for $H(\cdot\mid x)$-almost every $c$. Consequently,
    \begin{equation*}
        \bar{\mathcal S}_x(F)\ge \bar{\mathcal S}_x(P).
    \end{equation*}
    
    If, in addition, the localized scores are strictly proper, then the integrand vanishes if and only if
    \begin{equation*}
        P_{c,x}^{\flat}=F_{c,x}^{\flat}.
    \end{equation*}
    Since the integrand is nonnegative,
    \begin{equation*}
        \bar{\mathcal S}_x(F)=\bar{\mathcal S}_x(P)
    \end{equation*}
    if and only if
    \begin{equation*}
        P_{c,x}^{\flat}=F_{c,x}^{\flat}
    \end{equation*}
    for $H(\cdot\mid x)$-almost every $c$.
    
    It remains to relate equality of the induced censored laws to the latent joint CDF. Fix $t=(t_1,\ldots,t_k)\in\mathcal I_x$.

    By definition of $\mathcal I_x$, the set of censoring configurations satisfying $c_j>t_j$ for every $j$ has positive $H(\cdot\mid x)$-measure (or, for shared censoring, $c>\max_j t_j$).
    
    Since equality of the induced censored laws holds $H(\cdot\mid x)$-almost everywhere, there exists such a censoring configuration $c$ for which $P_{c,x}^{\flat}=F_{c,x}^{\flat}$.
    
    For this $c$, $\{\psi_c^\flat(T)\le t\} = \{T\le t\}$ where inequalities are interpreted componentwise. Hence equality of the induced censored laws implies
    \begin{equation*}
        \mathbb P_{Z\sim F_x}(Z\le t)
        =
        \mathbb P_{T\sim P_x}(T\le t).
    \end{equation*}
    Since $t\in\mathcal I_x$ was arbitrary, the latent joint CDFs agree throughout $\mathcal I_x$. If $\mathcal I_x=\mathbb R_+^k$, equality of the joint CDFs implies equality of the full latent probability laws.
\end{proof}

\subsection{Population optimality of censored engression}

\begin{proposition}[Population optimality of censored engression]
    \label{prop:censored-engression}
    Assume conditional independent censoring, $T\perp C\mid X$, and assume that $\Pi(dc\mid Y,\Delta,X)$ is the true forecast-independent conditional law of the censoring time given $(Y, \Delta, X)$. Assume that, for $H(\cdot\mid X=x)$-almost every censoring time $c$, the fixed-$c$ censored energy score is well-defined and strictly proper for the corresponding censored law. Suppose there exists $g_0\in\mathcal M$ such that
    \begin{equation*}
        g_0(x,\varepsilon)\sim P(T\mid X=x)
        \qquad
        \text{for }P_X\text{-almost every }x.
    \end{equation*}
    Let $\widetilde g \in \arg\min_{g\in\mathcal M} \mathcal R_{\mathrm{cens\text{-}ES}}(g)$. Then
    \begin{equation*}
        \mathcal L(\psi_c(\widetilde g(x,\varepsilon))\mid X=x)
        =
        \mathcal L(\psi_c(T)\mid X=x)
    \end{equation*}
    for $H(\cdot\mid X=x)$-almost every $c$ and $P_X$-almost every $x$. Consequently, the joint conditional CDF induced by $\widetilde g(x,\varepsilon)$ agrees with that of $T\mid X=x$ for every $t\in\mathcal I_x$. If $\mathcal I_x=\mathbb R_+^k$, the full conditional laws agree.
\end{proposition}

\begin{proof}
    Fix $x$ and a censoring time $c$. By assumption, the fixed-$c$ censored energy score is strictly proper for the censored law $\mathcal L(\psi_c(T)\mid X=x)$. Hence, the corresponding fixed-$c$ expected loss is minimized when
    \begin{equation*}
        \mathcal L(\psi_c(g(x,\varepsilon))\mid X=x)
        =
        \mathcal L(\psi_c(T)\mid X=x).
    \end{equation*}

    Since $\Pi(dc\mid Y,\Delta,X)$ is the true forecast-independent conditional law of the censoring time given the observed data, the population objective is the expectation of the corresponding observation-level marginalized score. Under correct specification, $g_0$ induces the true latent law and therefore attains zero integrated divergence, so it minimizes the population objective. Any other population minimizer $\widetilde g$ must also attain zero integrated divergence. Since the integrand is nonnegative for $H(\cdot\mid X=x)$-almost every $c$, it follows that
    \begin{equation*}
        \mathcal L(\psi_c(\widetilde g(x,\varepsilon))\mid X=x)
        =
        \mathcal L(\psi_c(T)\mid X=x)
    \end{equation*}
    for $H(\cdot\mid X=x)$-almost every $c$ and $P_X$-almost every $x$. The joint-CDF identification claim follows by the same argument as in Proposition~\ref{prop:marginalized-properness}.
\end{proof}

\section{Specific score derivations}
\label{app:score-derivations}
In each example, we proceed in the same way: we start from the uncensored score, define a fixed-censoring localized score, marginalize over the censoring time to obtain an observed-data score, and then relate this observed-data score back to the original latent target.

\subsection{Logarithmic score}
\label{app:log-score-censoring}
For the logarithmic score, the marginalization step simplifies substantially. Under right censoring, the fixed-$c$ localized score depends on the censoring time only through the censored contribution, which is either independent of $c$ or observed exactly from the data.

For fixed censoring time $c$, the induced censored law on the observed space has a density-mass representation
\begin{equation*}
    q_c^F(y,\delta\mid x)
    =
    \begin{cases}
        f(y\mid x), & \delta=1,\; y\le c,\\
        S(c\mid x), & \delta=0,\; y=c,
    \end{cases}
\end{equation*}
where $f(\cdot\mid x)$ is the conditional density of $T$ and $1-F(t\mid x)=\mathbb P(T>t\mid X=x)$ is the corresponding survival function. Hence, the fixed-$c$ localized logarithmic score is
\begin{equation}
    S^{\flat}_{\log,c}(F;y,\delta,x)
    =
    -\log q_c^F(y,\delta\mid x)
    =
    -\delta \log f(y\mid x)-(1-\delta)\log (1 - F(c\mid x)).
    \label{eq:fixed-c-log-score}
\end{equation}

The marginalized score is
\begin{equation*}
    \bar S_{\log}(F;y,\delta,x)
    =
    \mathbb E\!\left[
        S^{\flat}_{\log,C}(F;y,\delta,x)
        \mid Y=y,\Delta=\delta,X=x
    \right].
\end{equation*}
We distinguish the two possible values of $\delta$.

If $\delta=1$, then conditioning on $Y=y,\Delta=1,X=x$ implies that $T=y$ and $C\ge y$.
By \eqref{eq:fixed-c-log-score},
\begin{equation*}
    S^{\flat}_{\log,C}(F;y,1,x)=-\log f(y\mid x),
\end{equation*}
which does not depend on $C$. Therefore, $\bar S_{\log}(F;y,1,x)=-\log f(y\mid x)$.

If $\delta=0$, then conditioning on $Y=y,\Delta=0,X=x$ implies that $C=y$. Again by \eqref{eq:fixed-c-log-score},
\begin{equation*}
    S^{\flat}_{\log,C}(F;y,0,x)=-\log (1-F(C\mid x))=-\log (1 - F(y\mid x)),
\end{equation*}
and thus $\bar S_{\log}(F;y,0,x)=-\log (1 - F(y\mid x))$.

Combining the two cases yields
\begin{equation*}
    \bar S_{\log}(F;y,\delta,x) = -\delta\log f(y\mid x)-(1-\delta)\log (1 - F(y\mid x)).
\end{equation*}

In discrete or mixed settings, a boundary convention is needed at $Y=C$. With the convention $\Delta=\mathbb 1\{T\le C\}$, the censored contribution is naturally written in terms of $\mathbb P(T>Y\mid X)$, i.e.\ $1 - F(Y\mid X)$.

\subsection{CRPS}

For the CRPS, the marginalization step is nontrivial. Under right censoring, the fixed-$c$ localized score depends on the censoring time through the upper integration limit, so averaging over the conditional law of $C$ is essential.

Recall that for a distribution function $F(\cdot\mid x)$ and observation $t$, the CRPS is
\begin{equation*}
    S_\mathrm{CRPS}(F;t,x)
    =
    \int_0^t F(s\mid x)^2\,ds
    +
    \int_t^\infty (1-F(s\mid x))^2\,ds.
\end{equation*}

For fixed censoring time $c$, the censored law is the law of $\min(T,c)$, whose distribution function is
\begin{equation*}
    F_c^{\flat}(s\mid x)
    =
    \begin{cases}
    F(s\mid x), & s<c,\\
    1, & s\ge c.
    \end{cases}
\end{equation*}
Hence, the fixed-$c$ localized CRPS is $S_{\mathrm{CRPS}, c}^{\flat}(F;y,\delta,x) = S_\mathrm{CRPS}(F_c^{\flat};y,x)$, which can be written as
\begin{equation}
    S_{\mathrm{CRPS}, c}^{\flat}(F;y,\delta,x) =
    \int_0^y F(s\mid x)^2\,ds + \delta\int_y^c (1-F(s\mid x))^2\,ds.
\label{eq:fixed-c-crps}
\end{equation}
If $\delta=1$, then $y\le c$ is observed exactly and the second term runs only up to the censoring time $c$, because $F_c^{\flat}(s\mid x)=1$ for $s\ge c$. If $\delta=0$, then necessarily $y=c$, and the second term vanishes. 

The marginalized CRPS is
\begin{equation*}
    \bar{S}_{\mathrm{CRPS}}(F;y,\delta,x)
    =
    \mathbb E\!\left[
    S_{\mathrm{CRPS}, C}^{\flat}(F;y,\delta,x)\mid Y=y,\Delta=\delta,X=x
    \right].
\end{equation*}

If $\delta=0$, then conditioning on $Y=y,\Delta=0,X=x$ implies that $C=y$, so
\begin{equation*}
    \bar{S}_{\mathrm{CRPS}}(F;y,0,x) = \int_0^y F(s\mid x)^2\,ds.
\end{equation*}

If $\delta=1$, then conditioning on $Y=y,\Delta=1,X=x$ implies that $C\ge y$. Using \eqref{eq:fixed-c-crps},
\begin{equation*}
    \bar{S}_{\mathrm{CRPS}}(F;y,1,x) = \int_0^y F(s\mid x)^2\,ds +
    \mathbb E\!\left[
    \int_y^C (1-F(s\mid x))^2\,ds
    \;\middle|\;
    C\ge y,X=x
    \right].
\end{equation*}

By Fubini's theorem,
\begin{equation*}
    \mathbb E\!\left[
    \int_y^C (1-F(s\mid x))^2\,ds
    \;\middle|\;
    C\ge y,X=x
    \right]
    =
    \int_y^\infty
    \mathbb P(C> s\mid C\ge y,X=x)\,
    (1-F(s\mid x))^2\,ds.
\end{equation*}

Writing $G(t^-\mid x):=\mathbb P(C\ge t \mid X=x)$, this becomes
\begin{equation*}
    \int_y^\infty \frac{G(s\mid x)}{ G(y^-\mid x)}(1-F(s\mid x))^2\,ds.
\end{equation*}

Therefore,
\begin{equation}
    \bar{S}_{\mathrm{CRPS}}(F;y,\delta,x)
    =
    \int_0^y F(s\mid x)^2\,ds
    +
    \delta\int_y^\infty
    \frac{ G(s\mid x)}{ G(y^-\mid x)}
    (1-F(s\mid x))^2\,ds.
    \label{eq:marginalized-crps-right-censoring}
\end{equation}
This yields the marginalized right-censored CRPS induced by our construction.

\subsection{Energy score}
For the energy score, the fixed-censoring localized score is most naturally written in its expectation form. Unlike the logarithmic score, the marginalization step does not simplify away in general, because the fixed-censoring score depends on the censoring configuration through both expectation terms. In the univariate case with $\beta=1$, the energy score reduces to the CRPS.

Let $\beta\in(0,2)$, and recall that for a predictive law $F(\cdot\mid x)$ on $\mathbb R_+^k$, the energy score is
\begin{equation*}
    S_{\mathrm{ES}, \beta}(F;t,x) = \mathbb E\!\left[\|Z-t\|^\beta\mid X=x\right] - \frac12\mathbb E\!\left[\|Z-Z'\|^\beta\mid X=x\right],
\end{equation*}
where $Z,Z'\stackrel{\text{iid}}{\sim}F(\cdot\mid x)$.

For fixed componentwise censoring threshold $c=(c_1,\dots,c_k)$, define
\begin{equation*}
    \psi_c(t) = (\min(t_1,c_1),\dots,\min(t_k,c_k)).
\end{equation*}
If $T\sim F(\cdot\mid X=x)$, the induced censored law is
\begin{equation*}
    F_c^{\flat}(\cdot\mid x):=\mathcal L(\psi_c(T)\mid X=x),
\end{equation*}
and the observed censored outcome is $y =\psi_c(t)$. Hence the fixed-$c$ localized energy score is
\begin{align}
    S^\flat_{\mathrm{ES},\beta,c}(F;y,x)
    =
    \mathbb E[\|\psi_c(Z)-y\|^\beta\mid X=x]
    -\frac12
    \mathbb E[\|\psi_c(Z)-\psi_c(Z')\|^\beta\mid X=x]
    \label{eq:fixed-c-energy-score}
\end{align}
where $Z,Z'\stackrel{\text{iid}}{\sim}F(\cdot\mid X=x)$.

If the censoring threshold is random, the marginalized energy score is obtained by averaging \eqref{eq:fixed-c-energy-score} over the conditional law of $C$ given the observed data:
\begin{equation*}
    \bar S_{\mathrm{ES},\beta}(F;Y,\Delta,X)
    =
    \mathbb E\!\left[
    S^\flat_{\mathrm{ES},\beta,C}(F;Y,X)
    \mid Y,\Delta,X
    \right].
\end{equation*}
In contrast to the logarithmic score and the CRPS, no further simplification is available in general, since the fixed-$c$ localized score depends explicitly on $c$ through both expectation terms.

\subsection{Brier score}

Let $\tau\ge 0$ denote a fixed evaluation horizon. The uncensored Brier score at horizon $\tau$ is
\begin{equation}
    S_{\mathrm{BS},\tau}(F;t,x)
    =
    \bigl(F(\tau\mid x)-\mathbb 1\{t\le \tau\}\bigr)^2.
    \label{eq:uncensored-brier}
\end{equation}

Note that we can do similar derivations for the binomial log-likelihood for $\tau\ge 0$.

For fixed censoring time $c$, let $Y_c:=\min(T,c)$, $F_c^{\flat}:=\mathcal L(Y_c\mid X=x)$. The natural localized Brier score at horizon $\tau$ is the ordinary Brier score applied to the censored distribution:
\begin{equation}
    S_{\mathrm{BS},\tau,c}^{\flat}(F;t,x)
    :=
    \bigl(F_c^{\flat}(\tau\mid x)-\mathbb 1\{\min(t,c)\le \tau\}\bigr)^2.
    \label{eq:fixed-c-brier-direct}
\end{equation}
Since
\begin{equation*}
    F_c^{\flat}(\tau\mid x)
    =
    \mathbb P(\min(T,c)\le \tau\mid X=x)
    =
    \begin{cases}
    F(\tau\mid x), & \tau<c,\\
    1, & \tau\ge c,
    \end{cases}
\end{equation*}
the score can be written more explicitly as
\begin{equation}
    S_{\mathrm{BS},\tau,c}^{\flat}(F;t,x)
    =
    \begin{cases}
        \bigl(F(\tau\mid x)-\mathbb 1\{t\le \tau\}\bigr)^2, & \tau<c,\\[1mm]
        0, & \tau\ge c.
    \end{cases}
    \label{eq:fixed-c-brier-piecewise}
\end{equation}
Thus, for each fixed $c$, the localized Brier score coincides with the ordinary Brier score on the identifiable region $\tau<c$ and becomes vacuous beyond it. In particular, because the ordinary Brier score is strictly proper for Bernoulli outcomes, \eqref{eq:fixed-c-brier-direct} is strictly proper for the censored Bernoulli target $\mathbb 1\{\min(T,c)\le \tau\}$, and therefore locally proper for the latent law on horizons $\tau<c$.

We now marginalize over the censoring time. Let $Y=\min(T,C)$, $\Delta=\mathbb 1\{T\le C\}$, and define
\begin{equation*}
    \bar S_{\mathrm{BS},\tau}(F;Y,\Delta,X)
    :=
    \mathbb E\!\left[
    S_{\mathrm{BS},\tau,C}^{\flat}(F;Y,X)
    \mid Y,\Delta,X
    \right].
\end{equation*}

Using \eqref{eq:fixed-c-brier-piecewise}, we obtain the observed-data score by distinguishing three cases. If $Y>\tau$, then necessarily $C>\tau$ and $T>\tau$, so
\begin{equation*}
    \bar S_{\mathrm{BS},\tau}(F;Y,\Delta,X)=F(\tau\mid X)^2.
\end{equation*}

If $Y\le \tau$ and $\Delta=0$, then $C=Y\le \tau$, hence $\tau\ge C$ and
\begin{equation*}
    \bar S_{\mathrm{BS},\tau}(F;Y,0,X)=0.
\end{equation*}
If $Y\le \tau$ and $\Delta=1$, then $T=Y\le \tau$ and $C\ge Y$. In this case,
\begin{equation*}
    S^\flat_{\mathrm{BS},\tau,C}(F;Y,X)
    =
    \begin{cases}
    (1-F(\tau\mid X))^2, & C>\tau,\\
    0, & Y\le C\le \tau.
    \end{cases}
\end{equation*}

Therefore,
\begin{equation*}
    \bar S_{\mathrm{BS},\tau}(F;Y,1,X)
    =
    \mathbb P(C>\tau\mid C\ge Y,X)\,(1-F(\tau\mid X))^2.
\end{equation*}
Writing
\begin{equation*}
    G(t\mid x):=\mathbb P(C>t\mid X=x), \qquad G(t^-\mid x):=\mathbb P(C\ge t\mid X=x),
\end{equation*}
this becomes
\begin{equation*}
    \bar S_{\mathrm{BS},\tau}(F;Y,1,X)
    =
    \frac{G(\tau\mid X)}{G(Y^-\mid X)}(1-F(\tau\mid X))^2.
\end{equation*}
Combining the three cases yields
\begin{equation}
    \bar S_{\mathrm{BS},\tau}(F;Y,\Delta,X)
    =
    \mathbb 1\{Y>\tau\}F(\tau\mid X)^2
    +
    \Delta\,\mathbb 1\{Y\le \tau\}
    \frac{G(\tau\mid X)}{G(Y^-\mid X)}
    (1-F(\tau\mid X))^2.
    \label{eq:marginalized-brier-right}
\end{equation}
This is the observed-data version of the localized proper Brier score.

A convenient rescaling recovers the usual IPCW Brier score: dividing \eqref{eq:marginalized-brier-right} by $G(\tau \mid X)$ yields
\begin{equation*}
    \frac{\bar S_{\mathrm{BS},\tau}(F;Y,\Delta,X)}{G(\tau \mid X)} = \frac{\Delta \mathbb 1 \{ Y \le \tau\} (1 - F(\tau \mid X))^2}{G(Y^- \mid X)} + \frac{\mathbb 1 \{ Y > \tau \} F(\tau \mid X)^2}{G(\tau \mid X)},
\end{equation*}
which is exactly the standard IPCW Brier score \citep{graf_assessment_1999, gerds_consistent_2006}. Since this rescales the score by a strictly positive factor that does not depend on $F$, it preserves the same minimizer and hence the same propriety characteristics on the identifiable region.

\subsubsection{Integrated Brier score}
\label{app:ibs-right-censoring}
The integrated version fits especially naturally into the localization framework. Define the fixed-$c$ localized integrated Brier score by
\begin{equation}
    S_{\mathrm{IBS},c}^{\flat}(F;t,x)
    :=
    \int_0^\infty S_{\mathrm{BS},\tau,c}^{\flat}(F;t,x)\,d\tau.
    \label{eq:fixed-c-ibs-def}
\end{equation}
Using \eqref{eq:fixed-c-brier-piecewise}, this reduces to
\begin{equation}
    S_{\mathrm{IBS},c}^{\flat}(F;t,x)
    =
    \int_0^c
    \bigl(F(\tau\mid x)-\mathbb 1\{t\le \tau\}\bigr)^2
    \,d\tau.
    \label{eq:fixed-c-ibs}
\end{equation}
Thus the fixed-censoring localized IBS is simply the ordinary integrated Brier score truncated to the identifiable region $[0,c)$.

Marginalizing over the censoring time gives
\begin{equation*}
    \bar S_{\mathrm{IBS}}(F;Y,\Delta,X)
    =
    \mathbb E\!\left[
    S_{\mathrm{IBS},C}^{\flat}(F;Y,X)\mid Y,\Delta,X
    \right].
\end{equation*}

By Fubini's theorem and \eqref{eq:marginalized-brier-right},
\begin{align}
    \bar S_{\mathrm{IBS}}(F;Y,\Delta,X)
    &=
    \int_0^\infty
    \bar S_{\mathrm{BS},\tau}(F;Y,\Delta,X)\,d\tau \notag\\
    &=
    \int_0^Y F(\tau\mid X)^2\,d\tau
    +
    \Delta\int_Y^\infty
    \frac{G(\tau\mid X)}{G(Y^-\mid X)}
    (1-F(\tau\mid X))^2\,d\tau.
    \label{eq:marginalized-ibs-right}
\end{align}
This is the integrated observed-data score induced by the localized proper Brier construction.

In particular, \eqref{eq:marginalized-ibs-right} has the same structural form as the marginalized CRPS: an integral of squared CDF error up to the observed time, plus a tail term weighted by the conditional survival of the censoring time. In the univariate continuous setting, this is not merely a formal similarity: the integrated Brier score and the CRPS coincide, so the marginalized IBS and the marginalized CRPS are the same score up to the usual boundary convention.

\subsection{Pinball loss}

Let $\alpha\in(0,1)$ denote a quantile level, and let
\begin{equation*}
    q_\alpha(F\mid x):=\inf\{t\in\mathbb R_+:F(t\mid x)\ge \alpha\}
\end{equation*}
be the corresponding conditional $\alpha$-quantile implied by the forecast $F(\cdot\mid x)$. The uncensored quantile score is
\begin{equation}
    S_{\mathrm{Q},\alpha}(F;t,x)
    =
    \bigl(\alpha-\mathbb 1\{t<q_\alpha(F\mid x)\}\bigr)\bigl(t-q_\alpha(F\mid x)\bigr).
    \label{eq:uncensored-quantile-score}
\end{equation}
Equivalently,
\begin{equation*}
    S_{\mathrm{Q},\alpha}(F;t,x)
    =
    \alpha\bigl(t-q_\alpha(F\mid x)\bigr)_+
    +
    (1-\alpha)\bigl(q_\alpha(F\mid x)-t\bigr)_+.
\end{equation*}
This score is strictly consistent for the $\alpha$-quantile functional.

For fixed censoring time $c$, let $Y_c:=\min(T,c)$, $F_c^{\flat}:=\mathcal L(Y_c\mid X=x)$. The natural localized quantile score is the ordinary quantile score applied to the censored law:
\begin{equation}
    S_{\mathrm{Q},\alpha,c}^{\flat}(F;t,x)
    :=
    \bigl(\alpha-\mathbb 1\{\min(t,c)<q_\alpha(F_c^{\flat}\mid x)\}\bigr)
    \bigl(\min(t,c)-q_\alpha(F_c^{\flat}\mid x)\bigr).
    \label{eq:fixed-c-quantile-direct}
\end{equation}
Since
\begin{equation*}
    q_\alpha(F_c^{\flat}\mid x)=\min\!\bigl(q_\alpha(F\mid x),c\bigr),
\end{equation*}
this can be written as
\begin{equation}
    S_{\mathrm{Q},\alpha,c}^{\flat}(F;t,x)
    =
    \bigl(\alpha-\mathbb 1\{\min(t,c)<\min(q_\alpha(F\mid x),c)\}\bigr)
    \bigl(\min(t,c)-\min(q_\alpha(F\mid x),c)\bigr).
    \label{eq:fixed-c-quantile}
\end{equation}
Thus the fixed-censoring localized score is simply the ordinary quantile score for the censored
outcome $\min(T,c)$ and the censored quantile $\min(q_\alpha(F\mid x),c)$.

Let $Y=\min(T,C)$, $\Delta=\mathbb 1\{T\le C\}$, and define
\begin{equation*}
    \bar S_{\mathrm{Q},\alpha}(F;Y,\Delta,X)
    :=
    \mathbb E\!\left[
    S_{\mathrm{Q},\alpha,C}^{\flat}(F;Y,X)\mid Y,\Delta,X
    \right].
\end{equation*}

Write
\begin{equation*}
    q_\alpha(X):=q_\alpha(F\mid X),
    \qquad
    G(t\mid x):=\mathbb P(C>t\mid X=x),
    \qquad
    G(t^-\mid x):=\mathbb P(C\ge t\mid X=x).
\end{equation*}

If $\Delta=0$, then $C=Y$, so $q_\alpha(F_C^{\flat}\mid X)=\min(q_\alpha(F\mid X),Y),$ and therefore
\begin{equation*}
    \bar S_{\mathrm{Q},\alpha}(F;Y,0,X)
    =
    \bigl(\alpha-\mathbb 1\{Y<\min(q_\alpha(F\mid X),Y)\}\bigr)
    \bigl(Y-\min(q_\alpha(F\mid X),Y)\bigr)
    =
    \alpha\bigl(Y-q_\alpha(F\mid X)\bigr)_+.
\end{equation*}

If $\Delta=1$, then $T=Y$ and $C\ge Y$. In this case,
\begin{equation*}
    \bar S_{\mathrm{Q},\alpha}(F;Y,1,X)
    =
    \mathbb E\!\left[
    \bigl(\alpha-\mathbb 1\{Y<\min(q_\alpha(F\mid X),C)\}\bigr)
    \bigl(Y-\min(q_\alpha(F\mid X),C)\bigr)
    \mid C\ge Y,X
    \right].
\end{equation*}

If $q_\alpha(F\mid X)\le Y$, then $\min(q_\alpha(F\mid X),C)=q_\alpha(F\mid X)$ for every $C\ge Y$, so
\begin{equation*}
    \bar S_{\mathrm{Q},\alpha}(F;Y,1,X)
    =
    \alpha\bigl(Y-q_\alpha(F\mid X)\bigr).
\end{equation*}
If $q_\alpha(F\mid X)>Y$, then
\begin{equation*}
    \bigl(\alpha-\mathbb 1\{Y<\min(q_\alpha(F\mid X),C)\}\bigr)
    \bigl(Y-\min(q_\alpha(F\mid X),C)\bigr)
    =
    (1-\alpha)\bigl(\min(q_\alpha(F\mid X),C)-Y\bigr),
\end{equation*}
and hence
\begin{align*}
    \bar S_{\mathrm{Q},\alpha}(F;Y,1,X)
    &=
    \frac{1-\alpha}{G(Y^-\mid X)}
    \mathbb E\!\left[
    \bigl(\min(q_\alpha(F\mid X),C)-Y\bigr)\mathbb 1\{C\ge Y\}
    \mid X
    \right] \\
    &=
    \frac{1-\alpha}{G(Y^-\mid X)}
    \int_Y^{q_\alpha(F\mid X)} G(t\mid X)\,dt.
\end{align*}
Combining the two cases yields
\begin{equation}
    \bar S_{\mathrm{Q},\alpha}(F;Y,\Delta,X)
    =
    \alpha\bigl(Y-q_\alpha(F\mid X)\bigr)_+
    +
    \Delta\,\frac{1-\alpha}{G(Y^-\mid X)}
    \int_Y^{q_\alpha(F\mid X)} G(t\mid X)\,dt,
    \label{eq:marginalized-quantile-right}
\end{equation}
with the convention that the integral is zero whenever $q_\alpha(X)\le Y$.

\section{Additional experimental details and results}
\label{app:additional-experimental-results}

\subsection{Computational cost and stability}
\label{app:compute-resources}
The synthetic censored-engression experiments were run on an Apple M1 Pro with 32\,GB unified memory using the MPS backend. Table~\ref{tab:compute-stability} reports wall-clock training time, peak process memory, and epochs to early stopping over five independent repetitions.

\begin{table}[t]
\centering
\small
\caption{Computational diagnostics for the synthetic log-normal
censored-engression experiments. Censoring rate, training time, and
epochs are mean $\pm$ standard deviation over five repetitions.
Peak memory is the largest per-process resident-set size observed
across the five runs.}
\label{tab:compute-stability}
\begin{tabular}{llrrrr}
\toprule
$k$ & Censoring & Censored (\%) & Time (s) & Peak MiB & Epochs \\
\midrule
2  & Administrative       & $15.9\pm0.4$ & $49.0\pm5.6$   & 457.5 & $31.0\pm3.7$ \\
2  & Uniform              & $35.7\pm0.2$ & $51.5\pm6.7$   & 462.5 & $35.0\pm5.2$ \\
2  & Conditional uniform  & $43.1\pm0.3$ & $46.9\pm6.4$   & 463.3 & $32.0\pm4.9$ \\
5  & Administrative       & $17.8\pm0.2$ & $51.3\pm5.1$   & 458.2 & $30.0\pm3.5$ \\
5  & Uniform              & $36.5\pm0.3$ & $73.8\pm7.4$   & 458.1 & $45.0\pm5.0$ \\
5  & Conditional uniform  & $43.7\pm0.2$ & $66.2\pm8.0$   & 458.2 & $39.0\pm5.1$ \\
10 & Administrative       & $22.7\pm0.1$ & $92.7\pm10.2$  & 445.5 & $45.0\pm5.2$ \\
10 & Uniform              & $38.7\pm0.2$ & $106.7\pm16.2$ & 472.5 & $52.0\pm8.7$ \\
10 & Conditional uniform  & $45.2\pm0.2$ & $91.2\pm5.5$   & 465.3 & $44.0\pm2.9$ \\
\bottomrule
\end{tabular}
\end{table}

All runs reached the early-stopping criterion. We observed no fitting errors, non-finite training losses, or non-finite validation-selection metrics. Training remained below two minutes per run in every synthetic setting. Runtime increased with event dimension, whereas peak process memory remained between approximately 445 and 473\,MiB. Heavier censoring reduced statistical information and sometimes increased the number of epochs required, but did not produce optimization failures.

For the ICU experiment, training used 8 CPU cores of an Intel Xeon Platinum 8168 at 2.70\,GHz, 64\,GB system memory, and one NVIDIA Tesla V100-SXM3 GPU with 32\,GB memory. Final model training required $8{,}392$ seconds (139.9 minutes). Peak process resident-set size was $33{,}154$\,MiB and peak CUDA-reserved memory was 490\,MiB. Training times include optimization and validation passes but exclude preprocessing and final test scoring.

\subsection{Forecast-evaluation simulation setup}
\label{app:forecast-evaluation-setup}

This subsection gives the exact simulation design for the forecast-evaluation experiment reported in Section~\ref{sec:forecast-evaluation}. It specifies the latent event-time law, candidate forecasts, censoring regimes, discretization grid, and scoring rules used in the comparison.

\paragraph{Latent event-time model and grid.}
We simulate covariates $X=(X_1,X_2,X_3)\sim N(0,I_3)$ and latent event times from
\begin{equation*}
    T\mid X=x \sim \mathrm{Weibull}(k,\lambda(x)),
    \qquad
    k=1.5,
    \qquad
    \log \lambda(x)=0.3+0.8x_1-0.5x_2+0.3x_3 .
\end{equation*}
All reported results are based on a test sample of size $N=1000$. To match the discrete forecast-based scores, we represent forecasts on a grid
\begin{equation*}
    0=\zeta_0<\zeta_1<\cdots<\zeta_B=z_{\max},
    \qquad B=50,
\end{equation*}
where $z_{\max}=20.5471$ is the empirical $0.995$-quantile of a large pilot sample from the oracle distribution. We also evaluate continuous versions of our CRPS, pinball, Brier, and log scores before discretization when a continuous forecast representation is available; however, for conciseness, we omit them from Table \ref{tab:oracle-ranks}, as they have the same ranking as their discrete counterparts. Note that their scores are lower, since the DGP is continuous and the discretization is an approximation. 

\paragraph{Censoring regimes.}
We consider four regimes.

\emph{Regime A: administrative censoring.}
The censoring time is deterministic, $C\equiv c_{\mathrm{admin}}$, with $c_{\mathrm{admin}}=0.9833$ in the realized simulation, giving an event rate of $50.00\%$. Because the censoring time is known exactly, we evaluate this regime using localized scores.

\emph{Regime B: random independent censoring.}
The censoring time is sampled independently of $T$ and $X$:
\begin{equation*}
    C\sim \mathrm{Uniform}(0,C_{\max}),
    \qquad
    C_{\max}=0.4z_{\max}=8.2188 .
\end{equation*}
The event rate is $79.70\%$. Since the censoring law is marginal and forecast-independent, we evaluate this regime with marginalized scores using the true marginal censoring distribution.

\emph{Regime C: conditionally independent censoring.}
The censoring time depends on covariates but remains conditionally independent of \(T\):
\begin{equation*}
    C\mid X=x \sim \mathrm{Weibull}(k,\mu_C(x)),
    \qquad
    k=1.5,
    \qquad
    \log \mu_C(x)=0.2-0.3x_1+0.4x_3 .
\end{equation*}
The event rate is $48.60\%$. We evaluate this regime using marginalized scores with the true conditional censoring law.

\emph{Regime D: targeted forecast-dependent weights stress test.}
Let $j=\lfloor B/2\rfloor=25$, so that $\zeta_{j-1}=9.8626$ and $\zeta_j=10.2736$. Define
\begin{equation*}
    a_j=\zeta_{j-1}+0.25(\zeta_j-\zeta_{j-1})=9.9653,
    \qquad
    b_j=\zeta_{j-1}+0.50(\zeta_j-\zeta_{j-1})=10.0681 .
\end{equation*}
The oracle event distribution places probability $\frac12$ on $(b_j,\zeta_j]$ and probability $\frac12$ on $(\zeta_{B-1},\zeta_B]$, where $\zeta_{B-1}=20.1362 $ and $\zeta_B=20.5471$. Censoring is generated by a mixture: with probability $0.6$, $C\sim\mathrm{Uniform}(\zeta_{j-1},a_j)$, so censoring occurs before the interior event mass; with probability $0.4$, $C$ is set to a late administrative value beyond $z_{\max}$, so the event is observed. The resulting event rate is $39.50\%$. This regime is designed to allow forecast-dependent completion to be gamed by concentrating mass in the far tail.

\paragraph{Candidate forecasts.}
In regimes A--C, we compare five forecasts. $F_0$ is the oracle Weibull forecast. $F_1$ is a shifted Weibull forecast with $\lambda_1(x)=e^{0.25}\lambda(x)$. $F_2$ is an overdispersed and right-shifted discrete forecast obtained by smoothing the oracle pmf with a symmetric kernel proportional to $(1,2,3,2,1)$ and shifting mass two bins to the right. $F_3$ is underdispersed and right-shifted, constructed by blending the oracle pmf with a narrow peaked template centered two bins to the right of the oracle median bin, using width $1.25$ and mixing weight $0.70$. $F_4$ is upper-tail-heavy, with $p_{4,i}(x)\propto p_{0,i}(x)\exp(i/B)$.

Regime D uses a targeted stress-test construction. There, the oracle $F_0^{(D)}$ places half its mass in an interior bin interval and half its mass in the last grid bin. It is compared with exploit forecasts $F_5(\epsilon)$, with $\epsilon\in\{10^{-3},5\times 10^{-3},10^{-2},5\times 10^{-2}\}$, that place only $\epsilon$ mass in the interior event interval and nearly all remaining mass in the final bin, up to a small numerical probability floor.

\paragraph{Scoring rules and evaluation protocol.}
As infeasible gold standards, we evaluate latent CRPS, pinball, Brier, and negative log-likelihood scores with respect to the unobserved true event times. On censored data, we evaluate the forecast-independent scores proposed in this paper: localized scores in regime A and marginalized scores in regimes B--D. 

For direct comparison with \citet{yanagisawa_proper_2023}, we thus evaluate their right-censored forecast-dependent weighted log score, multiclass Brier score, binary Brier score, pinball score, and ranked probability score (RPS). Each is evaluated either with oracle weights computed from $F_0$, or with plug-in weights computed from the forecast being scored. The binary Brier score is averaged over event-time percentiles, using the targeted grid threshold in regime D.

\subsection{Oracle rankings and raw score values}
\label{app:forecast-evaluation-score-values}

The main text summarizes the principal finding that forecast-independent scores preserve the oracle ordering, whereas forecast-dependent plug-in scores can induce ranking reversals. Here we report the complete oracle rankings and corresponding raw score values.

Table~\ref{tab:oracle-ranks} reports the rank of the oracle forecast for each scoring rule across regimes A--D. The latent gold-standard scores rank the oracle first in every regime, as do all of our localized and marginalized scores. The oracle-weight versions of the \citet{yanagisawa_proper_2023} scores also rank the oracle first, confirming that these scores behave as intended when the required weights are evaluated using the correct event-time law.

The failures arise when those weights are recomputed from the candidate forecast. The plug-in NLL ranks the oracle second in regimes A, C, and D; the plug-in RPS ranks it second in A and C and last in D; and the plug-in Brier variants also exhibit reversals in C or D. In contrast, the plug-in pinball score preserves the oracle ordering in these experiments.

Tables~\ref{tab:oracle-score-values-abc} and \ref{tab:oracle-score-values-d} report the corresponding raw mean scores. They show that the ranking reversals result from competing or exploit forecasts attaining lower plug-in scores than the oracle, whereas the forecast-independent scores continue to assign the oracle the lowest score.

\begin{table}[t]
\centering

\begin{minipage}[t]{0.64\textwidth}
\vspace{0pt}
\centering
\scriptsize
\setlength{\tabcolsep}{3.5pt}
\renewcommand{\arraystretch}{0.95}
\caption{
Raw mean scores for regimes A--C. 
Each entry gives the oracle forecast score, with the best competing forecast score in parentheses. 
Lower is better. 
Bold entries indicate ranking reversals, where the best competing forecast obtains a lower score than the oracle.
}
\label{tab:oracle-score-values-abc}
\begin{adjustbox}{max width=\linewidth}
\begin{tabular}{lccc}
\toprule
Scoring rule & A & B & C \\
\midrule
CRPS, latent & $0.7766\;(0.7905)$ & $0.7766\;(0.7905)$ & $0.7766\;(0.7905)$ \\
Pinball, latent & $0.3030\;(0.3057)$ & $0.3030\;(0.3057)$ & $0.3030\;(0.3057)$ \\
Brier, latent & $0.1016\;(0.1021)$ & $0.1016\;(0.1021)$ & $0.1016\;(0.1021)$ \\
NLL, latent & $1.0745\;(1.1389)$ & $1.0745\;(1.1389)$ & $1.0745\;(1.1389)$ \\
\midrule
NLL \citep{yanagisawa_proper_2023}, plug-in 
& $\mathbf{0.9841\;(0.9800)}$ & $1.5163\;(1.5171)$ & $\mathbf{0.9086\;(0.9056)}$ \\
Brier MC \citep{yanagisawa_proper_2023}, plug-in 
& $0.7721\;(0.7735)$ & $0.7689\;(0.7694)$ & $0.7676\;(0.7690)$ \\
Brier \citep{yanagisawa_proper_2023}, plug-in 
& $0.1041\;(0.1045)$ & $0.1034\;(0.1040)$ & $\mathbf{0.1020\;(0.1019)}$ \\
Pinball \citep{yanagisawa_proper_2023}, plug-in 
& $0.2947\;(0.3400)$ & $0.3286\;(0.3477)$ & $0.3935\;(0.4197)$ \\
RPS \citep{yanagisawa_proper_2023}, plug-in 
& $\mathbf{1.6839\;(1.4671)}$ & $1.6877\;(1.7565)$ & $\mathbf{1.6683\;(1.4206)}$ \\
\midrule
NLL \citep{yanagisawa_proper_2023}, oracle 
& $0.9841\;(0.9856)$ & $1.5163\;(1.5193)$ & $0.9086\;(0.9099)$ \\
Brier MC \citep{yanagisawa_proper_2023}, oracle 
& $0.7721\;(0.7726)$ & $0.7689\;(0.7695)$ & $0.7676\;(0.7680)$ \\
Brier \citep{yanagisawa_proper_2023}, oracle 
& $0.1041\;(0.1045)$ & $0.1034\;(0.1039)$ & $0.1020\;(0.1023)$ \\
Pinball \citep{yanagisawa_proper_2023}, oracle 
& $0.2947\;(0.3400)$ & $0.3286\;(0.3448)$ & $0.3935\;(0.4178)$ \\
RPS \citep{yanagisawa_proper_2023}, oracle 
& $1.6839\;(1.7096)$ & $1.6877\;(1.7098)$ & $1.6683\;(1.6938)$ \\
\midrule
Censored NLL & $0.3811\;(0.4036)$ & $0.7350\;(0.7702)$ & $0.3445\;(0.3688)$ \\
CRPS, ours & $0.1129\;(0.1133)$ & $0.3729\;(0.3753)$ & $0.1268\;(0.1278)$ \\
Pinball, ours & $0.0416\;(0.0417)$ & $0.1418\;(0.1429)$ & $0.0473\;(0.0476)$ \\
Brier, ours & $0.0457\;(0.0460)$ & $0.0873\;(0.0880)$ & $0.0494\;(0.0497)$ \\
\bottomrule
\end{tabular}
\end{adjustbox}
\vspace{-1mm}
\end{minipage}
\hfill
\begin{minipage}[t]{0.34\textwidth}
\vspace{0pt}
\centering
\scriptsize
\setlength{\tabcolsep}{4pt}
\renewcommand{\arraystretch}{0.95}
\caption{
Raw mean scores for regime D. 
Entries report the oracle score and the best exploit forecast score. 
Lower is better. 
Bold entries indicate ranking reversals, where an exploit forecast obtains a lower score than the oracle.
}
\label{tab:oracle-score-values-d}
\begin{adjustbox}{max width=\linewidth}
\begin{tabular}{lcc}
\toprule
Scoring rule & Oracle score & Best exploit score \\
\midrule
CRPS, latent & $2.6189$ & $4.8200$ \\
Pinball, latent & $0.9320$ & $1.8684$ \\
Brier, latent & $0.0250$ & $0.0455$ \\
NLL, latent & $-0.5448$ & $0.2944$ \\
\midrule
NLL \citep{yanagisawa_proper_2023}, plug-in 
& $\mathbf{0.6931}$ & $\mathbf{0.6716}$ \\
Brier MC \citep{yanagisawa_proper_2023}, plug-in 
& $\mathbf{0.5000}$ & $\mathbf{0.3683}$ \\
Brier \citep{yanagisawa_proper_2023}, plug-in 
& $\mathbf{0.0250}$ & $\mathbf{0.0184}$ \\
Pinball \citep{yanagisawa_proper_2023}, plug-in 
& $1.1289$ & $1.4673$ \\
RPS \citep{yanagisawa_proper_2023}, plug-in 
& $\mathbf{6.2500}$ & $\mathbf{4.6039}$ \\
\midrule
NLL \citep{yanagisawa_proper_2023}, oracle 
& $0.6931$ & $1.4235$ \\
Brier MC \citep{yanagisawa_proper_2023}, oracle 
& $0.5000$ & $0.8438$ \\
Brier \citep{yanagisawa_proper_2023}, oracle 
& $0.0250$ & $0.0422$ \\
Pinball \citep{yanagisawa_proper_2023}, oracle 
& $1.1289$ & $1.5317$ \\
RPS \citep{yanagisawa_proper_2023}, oracle 
& $6.2500$ & $10.5473$ \\
\midrule
Censored NLL & $-0.2050$ & $0.0832$ \\
CRPS, ours & $1.0195$ & $1.6975$ \\
Pinball, ours & $0.3741$ & $0.6843$ \\
Brier, ours & $0.0099$ & $0.0167$ \\
\bottomrule
\end{tabular}
\end{adjustbox}
\vspace{-1mm}
\end{minipage}

\end{table}

\subsection{Sensitivity to estimating the censoring distribution}
\label{app:censoring-estimation-sensitivity}

We use the same synthetic setup as in Section~\ref{sec:forecast-evaluation}, with the exact simulation details given in Appendix~\ref{app:forecast-evaluation-setup}, but now replace the true censoring law in regimes B and C by estimated censoring distributions. The estimators are fit by treating censoring as the event. We consider a Kaplan--Meier estimator and a pooled Weibull model. In regime B, these pooled estimators are well matched to the DGP because censoring is independent of both $T$ and $X$. In regime C, they are intentionally misspecified because they ignore the covariate dependence of $C\mid X$.

Table~\ref{tab:estimated-censoring-scores} reports the resulting raw score values for our marginalized CRPS, pinball, and Brier scores, using both discrete and continuous implementations where available. The rank summary is as follows. In regime B, both the Kaplan--Meier and pooled Weibull censoring estimators preserve the oracle ordering for all three discrete scores. In regime C, the pooled Kaplan--Meier estimator also ranks the oracle first for all three discrete scores. 
The pooled Weibull estimator preserves the oracle ranking for discrete CRPS and pinball, but the discrete Brier score ranks the near-oracle tail forecast slightly ahead of the oracle. This reflects ordinary censoring-model misspecification rather than forecast-dependent circularity: once the censoring estimator is fixed, all candidate forecasts are still evaluated by the same forecast-independent score.

\begin{table*}[h]
\centering
\small
\setlength{\tabcolsep}{3.5pt}
\renewcommand{\arraystretch}{1.08}
\caption{
Sensitivity of our scores to estimated censoring distributions. 
Entries are mean scores; lower is better. 
Regime B uses censoring independent of both \(T\) and \(X\). 
Regime C uses covariate-dependent censoring, but the censoring estimator is pooled and therefore ignores the covariate dependence. 
}
\label{tab:estimated-censoring-scores}
\begin{adjustbox}{max width=\textwidth}
\begin{tabular}{llcccccc}
\toprule
Regime / estimator & Forecast 
& CRPS disc. & CRPS cont. 
& Pinball disc. & Pinball cont. 
& Brier disc. & Brier cont. \\
\midrule

\multicolumn{8}{l}{\emph{Regime B, KM censoring estimator}} \\
& \(F_0\) oracle  & 0.4326 & 0.3748 & 0.1637 & 0.1420 & 0.1023 & 0.0875 \\
& \(F_1\) shifted & 0.4863 & 0.4002 & 0.1827 & 0.1524 & 0.1093 & 0.0926 \\
& \(F_2\) over    & 0.7960 & 0.6497 & 0.2943 & 0.2427 & 0.2072 & 0.1819 \\
& \(F_3\) under   & 0.6281 & 0.5095 & 0.2346 & 0.1931 & 0.1470 & 0.1269 \\
& \(F_4\) tail    & 0.4405 & 0.3766 & 0.1666 & 0.1432 & 0.1031 & 0.0881 \\

\midrule
\multicolumn{8}{l}{\emph{Regime B, Weibull censoring estimator}} \\
& \(F_0\) oracle  & 0.4326 & 0.3752 & 0.1640 & 0.1424 & 0.1028 & 0.0872 \\
& \(F_1\) shifted & 0.4874 & 0.4037 & 0.1837 & 0.1535 & 0.1096 & 0.0922 \\
& \(F_2\) over    & 0.7926 & 0.6456 & 0.2935 & 0.2416 & 0.2070 & 0.1808 \\
& \(F_3\) under   & 0.6248 & 0.5061 & 0.2333 & 0.1919 & 0.1472 & 0.1261 \\
& \(F_4\) tail    & 0.4416 & 0.3786 & 0.1674 & 0.1441 & 0.1035 & 0.0878 \\

\midrule
\multicolumn{8}{l}{\emph{Regime C, KM censoring estimator, pooled}} \\
& \(F_0\) oracle  & 0.2402 & 0.1295 & 0.0891 & 0.0482 & 0.1023 & 0.0496 \\
& \(F_1\) shifted & 0.2618 & 0.1365 & 0.0965 & 0.0510 & 0.1084 & 0.0516 \\
& \(F_2\) over    & 0.4316 & 0.2708 & 0.1578 & 0.0998 & 0.2099 & 0.1073 \\
& \(F_3\) under   & 0.3301 & 0.1845 & 0.1198 & 0.0679 & 0.1466 & 0.0716 \\
& \(F_4\) tail    & 0.2432 & 0.1303 & 0.0901 & 0.0486 & 0.1028 & 0.0499 \\

\midrule
\multicolumn{8}{l}{\emph{Regime C, Weibull censoring estimator, pooled}} \\
& \(F_0\) oracle  & 0.2399 & 0.1281 & 0.0888 & 0.0474 & 0.1279 & 0.0496 \\
& \(F_1\) shifted & 0.2604 & 0.1334 & 0.0956 & 0.0495 & 0.1285 & 0.0517 \\
& \(F_2\) over    & 0.4352 & 0.2701 & 0.1587 & 0.0991 & 0.2300 & 0.1080 \\
& \(F_3\) under   & 0.3317 & 0.1838 & 0.1200 & 0.0674 & 0.1772 & 0.0720 \\
& \(F_4\) tail    & 0.2424 & 0.1288 & 0.0896 & 0.0476 & 0.1273 & 0.0499 \\

\bottomrule
\end{tabular}
\end{adjustbox}
\end{table*}

\subsection{Synthetic censored-engression setup}
\label{app:synthetic-censored-engression}

This subsection gives the DGPs and implementation details for the synthetic multivariate censored-engression experiments in Section~\ref{sec:synthetic-censored-engression}. 
The experiments use simulated data so that the latent event times and the true data-generating conditional law are available for oracle evaluation.

\paragraph{Covariates.}
For each observation $i$, we draw
\[
    X_i=(X_{i1},X_{i2},X_{i3},X_{i4})^\top \sim N(0,I_4).
\]
The latent event-time vector has dimension $k$, where $k=2$ for the bivariate experiments and $k\in\{3,5,10\}$ for the scaling experiments.

\paragraph{Base coefficient functions.}
For both latent DGPs, we use the same covariate-dependent location and scale functions. 
Let
\[
    b = (-0.6,-0.2,0.2,0.6)^\top .
\]
For event coordinate $j=1,\ldots,k$, define
\[
    \mu_j(x)=0.35 + x^\top \beta_j,
    \qquad
    \sigma_j(x)=0.35 + 0.15\,\mathrm{sigmoid}(x^\top \gamma_j),
\]
where
\[
    \beta_j
    =
    (0.35 + 0.08(j-1))
    \bigl(
    \cos(b_1+0.7(j-1)),\ldots,\cos(b_4+0.7(j-1))
    \bigr)^\top ,
\]
and
\[
    \gamma_j
    =
    0.10
    \bigl(
    \sin(b_1-0.7(j-1)),\ldots,\sin(b_4-0.7(j-1))
    \bigr)^\top .
\]

\paragraph{Unimodal log-normal DGP.}
The original, easier DGP is a correlated log-normal model. 
Let $\varepsilon_i\sim N(0,\Sigma_\rho)$, where $\Sigma_\rho$ is the $k\times k$ equicorrelation matrix with diagonal entries $1$ and off-diagonal entries $\rho=0.45$. 
Conditional on $X_i=x$, the latent log-times are
\[
    \log T_{ij}
    =
    \mu_j(x)+\sigma_j(x)\varepsilon_{ij},
    \qquad j=1,\ldots,k.
\]
We use this DGP only as a sanity check in the appendix, with $k=2$.

\paragraph{Mixture log-normal DGP.}
The main experiments use a covariate-dependent two-regime mixture log-normal DGP. 
Let
\[
    \pi(x)
    =
    \mathrm{sigmoid}
    \left(
    0.8x_1 - 0.6x_2 + 0.4\sin(x_3) + 0.25(x_4^2-1)
    \right),
\]
and draw a latent regime indicator
\[
    Z_i\mid X_i=x \sim \mathrm{Bernoulli}(\pi(x)).
\]
The two regimes correspond to an early and a late event-time mode. 
Let $\ell_j$ be event-coordinate-specific loadings, evenly spaced between $0.85$ and $1.15$:
\[
    \ell_j = 0.85 + \frac{j-1}{k-1}(1.15-0.85),
    \qquad j=1,\ldots,k,
\]
with $\ell_1=1$ when $k=1$. 
The regime shift is
\[
    a_j = 1.15\,\ell_j.
\]

Let $\varepsilon_i^{(0)}\sim N(0,\Sigma_{0})$ and $\varepsilon_i^{(1)}\sim N(0,\Sigma_{1})$, independently, where $\Sigma_0$ and $\Sigma_1$ are $k\times k$ equicorrelation matrices with off-diagonal correlations $0.70$ and $0.20$, respectively. 
The early regime has stronger dependence and smaller marginal scale, while the late regime has weaker dependence and larger marginal scale:
\[
    \log T_{ij}^{(0)}
    =
    \mu_j(x) - a_j + 0.85\,\sigma_j(x)\varepsilon_{ij}^{(0)},
\]
\[
    \log T_{ij}^{(1)}
    =
    \mu_j(x) + a_j + 1.15\,\sigma_j(x)\varepsilon_{ij}^{(1)}.
\]
The latent event time is then
\[
    \log T_{ij}
    =
    (1-Z_i)\log T_{ij}^{(0)}
    +
    Z_i\log T_{ij}^{(1)}.
\]
This DGP induces a nonlinear, covariate-dependent, multimodal conditional distribution and is intentionally outside the Weibull-copula family used by the parametric likelihood baselines.

\paragraph{Censoring mechanisms.}
A single scalar censoring time $C_i$ is shared by all event coordinates. 
The observed outcomes are
\[
    Y_{ij}=\min(T_{ij},C_i),
    \qquad
    \Delta_{ij}=\mathbb{1}\{T_{ij}\le C_i\},
    \qquad j=1,\ldots,k.
\]
Thus, censoring is perfectly dependent across coordinates through the shared censoring variable $C_i$.

We consider three censoring mechanisms:
\begin{itemize}
    \item \emph{Administrative censoring:}
    \[
        C \equiv c_{\mathrm{admin}} = 3.0.
    \]

    \item \emph{Random independent censoring:}
    \[
        C \sim \mathrm{Uniform}(0,c_{\max}),
        \qquad
        c_{\max}=5.0.
    \]

    \item \emph{Conditionally independent censoring:}
    \[
        C\mid X=x \sim \mathrm{Uniform}(0,c_{\max}(x)),
    \]
    with
    \[
        c_{\max}(x)
        =
        5.0
        \left(
        0.45
        +
        0.70\,\mathrm{sigmoid}(0.45x_1-0.35x_2+0.20x_3)
        \right).
    \]
\end{itemize}
The administrative regime corresponds to fixed localized censoring. 
The two random censoring regimes require the marginalized censored-energy objective because, for uncensored observations, the realized censoring time is only known to exceed the largest observed event time.

\paragraph{Methods and baselines.}
We compare censored engression with oracle references, naive observed engression, and likelihood-based survival baselines. 
The \emph{DGP} row evaluates samples from the true conditional law and is not a fitted method. 
\emph{Latent engression} is an infeasible benchmark trained directly on latent event times $T$. 
\emph{Naive observed engression} trains the same generator architecture on $Y$ as if it were fully observed. 
\emph{Censored engression} trains the generator with the localized or marginalized censored energy score.

The likelihood-based baselines are fitted only on $(X,Y,\Delta)$ using the censored logarithmic score. 
We include independent Weibull, conditional MLP independent Weibull, Clayton copula Weibull, and conditional MLP Clayton copula Weibull baselines. 
For $k=2$, we additionally include Gumbel copula Weibull and joint discrete-grid likelihood baselines. 
The latter is included only in low dimension because its output size scales as $O(B^k)$ for $B$ grid bins.

The conditional Weibull and Clayton-copula baselines use neural networks to parameterize the conditional distribution, so these comparisons are not simply between a neural method and fixed parametric models. Rather, they compare an implicit sample-based representation with neural likelihood models constrained to prespecified Weibull and copula families.

\paragraph{Training setup.}
All engression models use the same noise-conditioned MLP generator with hidden width $128$, noise dimension $128$, depth $2$, and a softplus output transform. 
Training, validation, and test sizes are $4000$, $1000$, and $1000$, respectively. 
We report means and standard errors over five independent repetitions.

Engression models are trained for at most $200$ epochs with early stopping. 
For censored engression, we use $16$ latent samples per minibatch update. 
For evaluation, we use $1024$ latent samples. 
In the random censoring regimes, the marginalized score is approximated with $8$ censoring draws during training and $512$ censoring draws during evaluation.

The likelihood baselines are trained by minimizing their observed-data censored negative log-likelihood. 
Validation negative log-likelihood is used for early stopping. 
For the conditional MLP Weibull and conditional MLP Clayton baselines, we use hidden width $128$, depth $2$, batch size $512$, and AdamW optimization.

\subsection{Synthetic censored-engression results}
\label{app:synthetic-censored-engression-results}
Across the mixture log-normal experiments, censored engression achieves the lowest censored energy score among feasible methods across the reported censoring mechanisms and dimensions. The improvement over naive observed engression becomes particularly pronounced under random and covariate-dependent censoring, where treating the censored vector as fully observed produces a distorted observed-data distribution. For example, under uniform censoring the censored energy score decreases from $0.5524$ to $0.4904$ at $k=2$ and from $1.6402$ to $1.3382$ at $k=10$.

The likelihood-based baselines provide substantially stronger comparisons than naive observed training because they are fitted with the proper censored logarithmic score. The conditional MLP versions additionally allow the Weibull marginals and copula parameters to depend flexibly on the covariates, so the comparison is not simply neural versus parametric. At $k=2$ under uniform censoring, for example, the MLP Clayton--Weibull model obtains censored ES $0.5152$, compared with $0.4904$ for censored engression. The remaining gap is consistent with the mixture log-normal DGP lying outside the prespecified Weibull-copula family, whereas censored engression represents the conditional joint distribution implicitly through generated samples.

The latent energy score serves a different purpose. It evaluates recovery of the full latent event-time law, including regions that are only partially identified from the censored observations. Likelihood-based copula models can therefore obtain lower latent energy score than censored engression in some settings even when they have worse censored energy score. This does not contradict the observed-data comparison: censored engression is optimized for the proper censored-data criterion, while behavior beyond the identifiable region is not fully determined by the observed censored distribution.

Table~\ref{tab:synthetic-censored-engression-lognormal-k2} reports the corresponding $k=2$ results under the unimodal log-normal sanity-check DGP. This setting is more favorable to the Weibull-copula likelihood models. As expected, these models are highly competitive and can outperform censored engression on latent energy score in some settings, while censored engression remains competitive on censored energy score and continues to improve substantially over naive observed engression.

\begin{table}[t]
\centering
\scriptsize
\setlength{\tabcolsep}{3.5pt}
\renewcommand{\arraystretch}{0.95}
\caption{Synthetic multivariate survival experiment under a covariate-dependent mixture log-normal latent DGP with $k=2$ event times and one shared censoring time. Entries are mean $\pm$ standard error over five independent repetitions. Lower is better. \emph{DGP} denotes the true data-generating conditional law and is therefore an oracle reference, not a fitted method. \emph{Latent engression} is the infeasible benchmark trained directly on latent event times. Among feasible methods, the best result in each censoring regime is shown in bold.}
\label{tab:synthetic-censored-engression-mixture-k2}
\begin{tabular}{llcc}
\toprule
Censoring & Method & Censored ES & Latent ES \\
\midrule
Administrative & DGP & $0.6332 \pm 0.0061$ & $1.2699 \pm 0.0183$ \\
Administrative & Latent engression & $0.6424 \pm 0.0055$ & $1.2862 \pm 0.0182$ \\
\cmidrule(lr){1-4}
Administrative & Naive observed engression & $0.6444 \pm 0.0053$ & $1.3941 \pm 0.0180$ \\
Administrative & Censored engression & $\mathbf{0.6401 \pm 0.0053}$ & $\mathbf{1.3237 \pm 0.0193}$ \\
Administrative & Independent Weibull & $0.7850 \pm 0.0033$ & $1.4765 \pm 0.0165$ \\
Administrative & MLP independent Weibull & $0.7029 \pm 0.0044$ & $1.3772 \pm 0.0175$ \\
Administrative & Clayton--Weibull & $0.7586 \pm 0.0027$ & $1.4515 \pm 0.0175$ \\
Administrative & MLP Clayton--Weibull & $0.6745 \pm 0.0055$ & $1.3388 \pm 0.0190$ \\
Administrative & Gumbel--Weibull & $0.7567 \pm 0.0025$ & $1.4509 \pm 0.0172$ \\
Administrative & Joint grid likelihood & $0.7143 \pm 0.0069$ & $1.4761 \pm 0.0156$ \\
\midrule
Uniform & DGP & $0.4840 \pm 0.0067$ & $1.2468 \pm 0.0168$ \\
Uniform & Latent engression & $0.4907 \pm 0.0061$ & $1.2601 \pm 0.0169$ \\
\cmidrule(lr){1-4}
Uniform & Naive observed engression & $0.5524 \pm 0.0125$ & $1.4440 \pm 0.0155$ \\
Uniform & Censored engression & $\mathbf{0.4904 \pm 0.0073}$ & $\mathbf{1.2887 \pm 0.0207}$ \\
Uniform & Independent Weibull & $0.5975 \pm 0.0071$ & $1.4703 \pm 0.0198$ \\
Uniform & MLP independent Weibull & $0.5333 \pm 0.0054$ & $1.3532 \pm 0.0208$ \\
Uniform & Clayton--Weibull & $0.5781 \pm 0.0070$ & $1.4426 \pm 0.0200$ \\
Uniform & MLP Clayton--Weibull & $0.5152 \pm 0.0070$ & $1.3233 \pm 0.0207$ \\
Uniform & Gumbel--Weibull & $0.5773 \pm 0.0071$ & $1.4409 \pm 0.0194$ \\
Uniform & Joint grid likelihood & $0.5245 \pm 0.0091$ & $1.3586 \pm 0.0183$ \\
\midrule
Conditional uniform & DGP & $0.3811 \pm 0.0053$ & $1.2468 \pm 0.0168$ \\
Conditional uniform & Latent engression & $0.3868 \pm 0.0049$ & $1.2601 \pm 0.0169$ \\
\cmidrule(lr){1-4}
Conditional uniform & Naive observed engression & $0.4553 \pm 0.0104$ & $1.5009 \pm 0.0199$ \\
Conditional uniform & Censored engression & $\mathbf{0.3871 \pm 0.0055}$ & $\mathbf{1.3088 \pm 0.0236}$ \\
Conditional uniform & Independent Weibull & $0.4774 \pm 0.0056$ & $1.4879 \pm 0.0214$ \\
Conditional uniform & MLP independent Weibull & $0.4178 \pm 0.0048$ & $1.3698 \pm 0.0217$ \\
Conditional uniform & Clayton--Weibull & $0.4633 \pm 0.0053$ & $1.4640 \pm 0.0206$ \\
Conditional uniform & MLP Clayton--Weibull & $0.4051 \pm 0.0057$ & $1.3399 \pm 0.0198$ \\
Conditional uniform & Gumbel--Weibull & $0.4620 \pm 0.0053$ & $1.4628 \pm 0.0209$ \\
Conditional uniform & Joint grid likelihood & $0.4124 \pm 0.0066$ & $1.3582 \pm 0.0219$ \\
\bottomrule
\end{tabular}
\vspace{-1mm}
\end{table}
\begin{table}[t]
\centering
\scriptsize
\setlength{\tabcolsep}{3.5pt}
\renewcommand{\arraystretch}{0.95}
\caption{Synthetic multivariate survival experiment under a covariate-dependent mixture log-normal latent DGP with $k=3$ event times and one shared censoring time. Entries are mean $\pm$ standard error over five independent repetitions. Lower is better. \emph{DGP} denotes the true data-generating conditional law and is therefore an oracle reference, not a fitted method. \emph{Latent engression} is the infeasible benchmark trained directly on latent event times. Among feasible methods, the best result in each censoring regime is shown in bold.}
\label{tab:synthetic-censored-engression-mixture-k3}
\begin{tabular}{llcc}
\toprule
Censoring & Method & Censored ES & Latent ES \\
\midrule
Administrative & DGP & $0.9189 \pm 0.0055$ & $3.2238 \pm 0.0389$ \\
Administrative & Latent engression & $0.9457 \pm 0.0044$ & $3.2923 \pm 0.0416$ \\
\cmidrule(lr){1-4}
Administrative & Naive observed engression & $0.9486 \pm 0.0050$ & $4.5247 \pm 0.0667$ \\
Administrative & Censored engression & $\mathbf{0.9434 \pm 0.0039}$ & $3.7475 \pm 0.0427$ \\
Administrative & Independent Weibull & $1.1043 \pm 0.0067$ & $3.8241 \pm 0.0559$ \\
Administrative & MLP independent Weibull & $1.0809 \pm 0.0056$ & $4.0990 \pm 0.0955$ \\
Administrative & Clayton--Weibull & $0.9905 \pm 0.0045$ & $3.6364 \pm 0.0520$ \\
Administrative & MLP Clayton--Weibull & $0.9817 \pm 0.0019$ & $\mathbf{3.5207 \pm 0.0361}$ \\
\midrule
Uniform & DGP & $0.7451 \pm 0.0073$ & $3.1883 \pm 0.0268$ \\
Uniform & Latent engression & $0.7651 \pm 0.0059$ & $3.2389 \pm 0.0296$ \\
\cmidrule(lr){1-4}
Uniform & Naive observed engression & $0.9003 \pm 0.0089$ & $4.4224 \pm 0.0810$ \\
Uniform & Censored engression & $\mathbf{0.7584 \pm 0.0071}$ & $3.5878 \pm 0.0608$ \\
Uniform & Independent Weibull & $0.8922 \pm 0.0086$ & $3.6567 \pm 0.0421$ \\
Uniform & MLP independent Weibull & $0.8724 \pm 0.0085$ & $3.8675 \pm 0.1367$ \\
Uniform & Clayton--Weibull & $0.8059 \pm 0.0085$ & $3.4878 \pm 0.0338$ \\
Uniform & MLP Clayton--Weibull & $0.8014 \pm 0.0078$ & $\mathbf{3.4073 \pm 0.0359}$ \\
\midrule
Conditional uniform & DGP & $0.5965 \pm 0.0073$ & $3.1883 \pm 0.0268$ \\
Conditional uniform & Latent engression & $0.6137 \pm 0.0060$ & $3.2389 \pm 0.0296$ \\
\cmidrule(lr){1-4}
Conditional uniform & Naive observed engression & $0.7489 \pm 0.0113$ & $4.5448 \pm 0.0852$ \\
Conditional uniform & Censored engression & $\mathbf{0.6089 \pm 0.0061}$ & $3.6355 \pm 0.0626$ \\
Conditional uniform & Independent Weibull & $0.7155 \pm 0.0087$ & $3.6283 \pm 0.0399$ \\
Conditional uniform & MLP independent Weibull & $0.7006 \pm 0.0086$ & $3.7896 \pm 0.0984$ \\
Conditional uniform & Clayton--Weibull & $0.6463 \pm 0.0084$ & $3.4849 \pm 0.0396$ \\
Conditional uniform & MLP Clayton--Weibull & $0.6416 \pm 0.0080$ & $\mathbf{3.4174 \pm 0.0448}$ \\
\bottomrule
\end{tabular}
\vspace{-1mm}
\end{table}
\begin{table}[t]
\centering
\scriptsize
\setlength{\tabcolsep}{3.5pt}
\renewcommand{\arraystretch}{0.95}
\caption{Synthetic multivariate survival experiment under a covariate-dependent mixture log-normal latent DGP with $k=5$ event times and one shared censoring time. Entries are mean $\pm$ standard error over five independent repetitions. Lower is better. \emph{DGP} denotes the true data-generating conditional law and is therefore an oracle reference, not a fitted method. \emph{Latent engression} is the infeasible benchmark trained directly on latent event times. Among feasible methods, the best result in each censoring regime is shown in bold.}
\label{tab:synthetic-censored-engression-mixture-k5}
\begin{tabular}{llcc}
\toprule
Censoring & Method & Censored ES & Latent ES \\
\midrule
Administrative & DGP & $1.2108 \pm 0.0129$ & $5.3617 \pm 0.0205$ \\
Administrative & Latent engression & $1.3204 \pm 0.0272$ & $5.5121 \pm 0.0363$ \\
\cmidrule(lr){1-4}
Administrative & Naive observed engression & $1.2445 \pm 0.0129$ & $7.9679 \pm 0.0529$ \\
Administrative & Censored engression & $\mathbf{1.2258 \pm 0.0121}$ & $7.1165 \pm 0.0267$ \\
Administrative & Independent Weibull & $1.4678 \pm 0.0119$ & $6.1825 \pm 0.0446$ \\
Administrative & MLP independent Weibull & $1.4394 \pm 0.0112$ & $6.6426 \pm 0.1223$ \\
Administrative & Clayton--Weibull & $1.3084 \pm 0.0122$ & $5.9206 \pm 0.0317$ \\
Administrative & MLP Clayton--Weibull & $1.3140 \pm 0.0167$ & $\mathbf{5.8265 \pm 0.0393}$ \\
\midrule
Uniform & DGP & $0.9536 \pm 0.0100$ & $5.1394 \pm 0.0078$ \\
Uniform & Latent engression & $1.0477 \pm 0.0204$ & $5.3048 \pm 0.0275$ \\
\cmidrule(lr){1-4}
Uniform & Naive observed engression & $1.1433 \pm 0.0070$ & $7.6696 \pm 0.0984$ \\
Uniform & Censored engression & $\mathbf{0.9692 \pm 0.0110}$ & $6.6095 \pm 0.0641$ \\
Uniform & Independent Weibull & $1.1529 \pm 0.0132$ & $5.9272 \pm 0.0301$ \\
Uniform & MLP independent Weibull & $1.1326 \pm 0.0132$ & $6.1641 \pm 0.0837$ \\
Uniform & Clayton--Weibull & $1.0356 \pm 0.0117$ & $5.6937 \pm 0.0365$ \\
Uniform & MLP Clayton--Weibull & $1.0441 \pm 0.0120$ & $\mathbf{5.5976 \pm 0.0432}$ \\
\midrule
Conditional uniform & DGP & $0.7638 \pm 0.0105$ & $5.1394 \pm 0.0078$ \\
Conditional uniform & Latent engression & $0.8536 \pm 0.0192$ & $5.3048 \pm 0.0275$ \\
\cmidrule(lr){1-4}
Conditional uniform & Naive observed engression & $0.9440 \pm 0.0064$ & $7.8476 \pm 0.0988$ \\
Conditional uniform & Censored engression & $\mathbf{0.7754 \pm 0.0117}$ & $6.7030 \pm 0.0690$ \\
Conditional uniform & Independent Weibull & $0.9230 \pm 0.0137$ & $5.9580 \pm 0.0362$ \\
Conditional uniform & MLP independent Weibull & $0.9022 \pm 0.0137$ & $6.2551 \pm 0.0646$ \\
Conditional uniform & Clayton--Weibull & $0.8308 \pm 0.0124$ & $5.7525 \pm 0.0428$ \\
Conditional uniform & MLP Clayton--Weibull & $0.8392 \pm 0.0148$ & $\mathbf{5.6475 \pm 0.0527}$ \\
\bottomrule
\end{tabular}
\vspace{-1mm}
\end{table}
\begin{table}[t]
\centering
\scriptsize
\setlength{\tabcolsep}{3.5pt}
\renewcommand{\arraystretch}{0.95}
\caption{Synthetic multivariate survival experiment under a covariate-dependent mixture log-normal latent DGP with $k=10$ event times and one shared censoring time. Entries are mean $\pm$ standard error over five independent repetitions. Lower is better. \emph{DGP} denotes the true data-generating conditional law and is therefore an oracle reference, not a fitted method. \emph{Latent engression} is the infeasible benchmark trained directly on latent event times. Among feasible methods, the best result in each censoring regime is shown in bold.}
\label{tab:synthetic-censored-engression-mixture-k10}
\begin{tabular}{llcc}
\toprule
Censoring & Method & Censored ES & Latent ES \\
\midrule
Administrative & DGP & $1.6530 \pm 0.0092$ & $9.6954 \pm 0.0843$ \\
Administrative & Latent engression & $2.0170 \pm 0.0403$ & $10.2168 \pm 0.0699$ \\
\cmidrule(lr){1-4}
Administrative & Naive observed engression & $1.7181 \pm 0.0098$ & $17.3182 \pm 0.2313$ \\
Administrative & Censored engression & $\mathbf{1.6884 \pm 0.0107}$ & $15.7127 \pm 0.1847$ \\
Administrative & Independent Weibull & $2.0100 \pm 0.0088$ & $14.9379 \pm 0.6241$ \\
Administrative & MLP independent Weibull & $1.9585 \pm 0.0118$ & $13.3928 \pm 0.2939$ \\
Administrative & Clayton--Weibull & $1.8086 \pm 0.0074$ & $14.3108 \pm 0.5964$ \\
Administrative & MLP Clayton--Weibull & $1.8793 \pm 0.0186$ & $\mathbf{13.1867 \pm 0.2063}$ \\
\midrule
Uniform & DGP & $1.3131 \pm 0.0055$ & $9.4613 \pm 0.0914$ \\
Uniform & Latent engression & $1.6082 \pm 0.0190$ & $10.0111 \pm 0.1553$ \\
\cmidrule(lr){1-4}
Uniform & Naive observed engression & $1.6402 \pm 0.0228$ & $16.5122 \pm 0.3144$ \\
Uniform & Censored engression & $\mathbf{1.3382 \pm 0.0052}$ & $14.2576 \pm 0.3110$ \\
Uniform & Independent Weibull & $1.5901 \pm 0.0086$ & $13.9407 \pm 0.4351$ \\
Uniform & MLP independent Weibull & $1.5491 \pm 0.0066$ & $\mathbf{12.4447 \pm 0.1425}$ \\
Uniform & Clayton--Weibull & $1.4397 \pm 0.0073$ & $13.3828 \pm 0.4909$ \\
Uniform & MLP Clayton--Weibull & $1.4783 \pm 0.0200$ & $12.7285 \pm 0.5428$ \\
\midrule
Conditional uniform & DGP & $1.0562 \pm 0.0067$ & $9.4613 \pm 0.0914$ \\
Conditional uniform & Latent engression & $1.3390 \pm 0.0175$ & $10.0111 \pm 0.1553$ \\
\cmidrule(lr){1-4}
Conditional uniform & Naive observed engression & $1.3623 \pm 0.0180$ & $16.8903 \pm 0.3222$ \\
Conditional uniform & Censored engression & $\mathbf{1.0815 \pm 0.0055}$ & $14.6636 \pm 0.2549$ \\
Conditional uniform & Independent Weibull & $1.2751 \pm 0.0089$ & $13.7644 \pm 0.4093$ \\
Conditional uniform & MLP independent Weibull & $1.2406 \pm 0.0083$ & $12.7779 \pm 0.1484$ \\
Conditional uniform & Clayton--Weibull & $1.1574 \pm 0.0079$ & $13.2792 \pm 0.4230$ \\
Conditional uniform & MLP Clayton--Weibull & $1.1837 \pm 0.0091$ & $\mathbf{12.5259 \pm 0.4063}$ \\
\bottomrule
\end{tabular}
\vspace{-1mm}
\end{table}
\begin{table}[t]
\centering
\scriptsize
\setlength{\tabcolsep}{3.5pt}
\renewcommand{\arraystretch}{0.95}
\caption{Synthetic multivariate survival experiment under a regular log-normal latent DGP with $k=2$ event times and one shared censoring time. Entries are mean $\pm$ standard error over five independent repetitions. Lower is better. \emph{DGP} denotes the true data-generating conditional law and is therefore an oracle reference, not a fitted method. \emph{Latent engression} is the infeasible benchmark trained directly on latent event times. Among feasible methods, the best result in each censoring regime is shown in bold.}
\label{tab:synthetic-censored-engression-lognormal-k2}
\begin{tabular}{llcc}
\toprule
Censoring & Method & Censored ES & Latent ES \\
\midrule
Administrative & DGP & $0.4288 \pm 0.0040$ & $0.7176 \pm 0.0069$ \\
Administrative & Latent engression & $0.4305 \pm 0.0048$ & $0.7196 \pm 0.0063$ \\
\cmidrule(lr){1-4}
Administrative & Naive observed engression & $0.4416 \pm 0.0035$ & $0.8642 \pm 0.0135$ \\
Administrative & Censored engression & $\mathbf{0.4308 \pm 0.0041}$ & $0.7892 \pm 0.0123$ \\
Administrative & Independent Weibull & $0.4365 \pm 0.0039$ & $0.7268 \pm 0.0071$ \\
Administrative & MLP independent Weibull & $0.4380 \pm 0.0049$ & $0.7403 \pm 0.0081$ \\
Administrative & Clayton--Weibull & $0.4329 \pm 0.0038$ & $\mathbf{0.7225 \pm 0.0069}$ \\
Administrative & MLP Clayton--Weibull & $0.4360 \pm 0.0051$ & $0.7299 \pm 0.0091$ \\
Administrative & Gumbel--Weibull & $0.4340 \pm 0.0040$ & $0.7236 \pm 0.0071$ \\
Administrative & Joint grid likelihood & $0.4506 \pm 0.0037$ & $0.9149 \pm 0.0167$ \\
\midrule
Uniform & DGP & $0.3221 \pm 0.0056$ & $0.7175 \pm 0.0117$ \\
Uniform & Latent engression & $0.3238 \pm 0.0060$ & $0.7259 \pm 0.0119$ \\
\cmidrule(lr){1-4}
Uniform & Naive observed engression & $0.3975 \pm 0.0088$ & $0.9709 \pm 0.0211$ \\
Uniform & Censored engression & $0.3257 \pm 0.0055$ & $0.7886 \pm 0.0249$ \\
Uniform & Independent Weibull & $0.3275 \pm 0.0054$ & $0.7259 \pm 0.0127$ \\
Uniform & MLP independent Weibull & $0.3290 \pm 0.0063$ & $0.7421 \pm 0.0138$ \\
Uniform & Clayton--Weibull & $\mathbf{0.3251 \pm 0.0053}$ & $0.7229 \pm 0.0128$ \\
Uniform & MLP Clayton--Weibull & $0.3280 \pm 0.0064$ & $0.7395 \pm 0.0126$ \\
Uniform & Gumbel--Weibull & $0.3257 \pm 0.0055$ & $\mathbf{0.7225 \pm 0.0128}$ \\
Uniform & Joint grid likelihood & $0.3407 \pm 0.0057$ & $0.8191 \pm 0.0204$ \\
\midrule
Conditional uniform & DGP & $0.2672 \pm 0.0055$ & $0.7175 \pm 0.0117$ \\
Conditional uniform & Latent engression & $0.2687 \pm 0.0057$ & $0.7259 \pm 0.0119$ \\
\cmidrule(lr){1-4}
Conditional uniform & Naive observed engression & $0.3557 \pm 0.0098$ & $1.0458 \pm 0.0231$ \\
Conditional uniform & Censored engression & $\mathbf{0.2694 \pm 0.0057}$ & $0.7927 \pm 0.0198$ \\
Conditional uniform & Independent Weibull & $0.2717 \pm 0.0051$ & $0.7266 \pm 0.0128$ \\
Conditional uniform & MLP independent Weibull & $0.2729 \pm 0.0061$ & $0.7504 \pm 0.0147$ \\
Conditional uniform & Clayton--Weibull & $0.2699 \pm 0.0053$ & $0.7235 \pm 0.0125$ \\
Conditional uniform & MLP Clayton--Weibull & $0.2729 \pm 0.0064$ & $0.7478 \pm 0.0132$ \\
Conditional uniform & Gumbel--Weibull & $0.2703 \pm 0.0053$ & $\mathbf{0.7233 \pm 0.0129}$ \\
Conditional uniform & Joint grid likelihood & $0.2798 \pm 0.0057$ & $0.8203 \pm 0.0199$ \\
\bottomrule
\end{tabular}
\vspace{-1mm}
\end{table}

\clearpage

\subsection{Additional AKI use-case setup and results}
\label{app:aki-additional-results}

This subsection describes the ICU AKI use case used in Section~\ref{sec:aki-use-case}. The goal is methodological: we use a realistic, irregularly sampled ICU prediction problem to test whether a censoring-aware distributional objective improves observed-data predictive distributions. We therefore report processed cohort sizes and endpoint rates, but do not include a demographics table.

\begin{table}[ht]
\centering
\caption{KDIGO Acute Kidney Injury (AKI) stage definitions for serum creatinine (sCr) and urine output (UO) used in this work.}
\label{tab:kdigo-stages}
\begin{tabular}{p{1.1cm} p{4cm} p{4cm}}
\toprule
Stage & sCr criteria & UO criteria \\
\midrule

\textbf{Stage 1} 
& \begin{tabular}[t]{@{}l@{}}
    sCr $\geq$ $1.5-1.9 \times$ baseline, or\\
    increase in sCr $\geq 0.3$ mg/dL
  \end{tabular}
& \begin{tabular}[t]{@{}l@{}}
    UO $\le 0.5$ mL/kg/h for 6--12 h
  \end{tabular}
\\[2mm]
\midrule

\textbf{Stage 2} 
& \begin{tabular}[t]{@{}l@{}}
    sCr $\geq$ $2.0-2.9 \times$ baseline
  \end{tabular}
& \begin{tabular}[t]{@{}l@{}}
    UO $\le 0.5$ mL/kg/h for $\geq 12$ h
  \end{tabular}
\\[2mm]
\midrule

\textbf{Stage 3} 
& \begin{tabular}[t]{@{}l@{}}
    sCr $\geq 3.0 \times$ baseline, or \\
    sCr $\geq 4.0$ mg/dL
  \end{tabular}
& \begin{tabular}[t]{@{}l@{}}
    UO $\le 0.3$ mL/kg/h for $\geq 24$ h
  \end{tabular}
\\

\bottomrule
\end{tabular}
\end{table}

\paragraph{Cohort construction and prediction rows.}
We use MIMIC-IV ICU stays and construct one prediction row per eligible ICU stay and timestamp. Prediction rows are kept only when they occur within the first $14$ days after ICU admission and have positive remaining follow-up time before ICU exit (discharge alive or in-ICU death). The two event times are the remaining times from the prediction timestamp to KDIGO ~\citep{kdigo_kdigo_2012}  stage-2 AKI according to the creatinine criterion and according to the urine-output criterion, see Table \ref{tab:kdigo-stages}. The shared censoring time $C$ is the remaining time to ICU exit. For endpoint $j\in\{1,2\}$,
\begin{equation*}
    Y_j=\min(T_j,C),
    \qquad
    \Delta_j=\mathbb 1\{T_j\le C\},
\end{equation*}
where $T_j$ is the latent time to the corresponding AKI definition. Times are transformed during training as
\begin{equation*}
    t \mapsto \frac{\log(1+t)}{\log(1+q_{0.95})},
\end{equation*}
where $q_{0.95}$ is the $95$th percentile of the training ICU-exit-time distribution. All reported localized energy scores and fixed-horizon metrics are computed after transforming sampled times back to hours.

The split is chronological at the ICU-stay level. Stays from anchor-year groups 2008--2016, plus one-half of the 2017--2019 stays, are used for training; the remaining 2017--2019 stays are used for validation; 2020--2022 stays are held out for testing. This yields $4{,}831{,}810$ training rows, $575{,}500$ validation rows, and $835{,}661$ test rows. The training row-level event rates are $18.3\%$ for creatinine-defined AKI and $39.7\%$ for urine-output-defined AKI; however, note that this is an artificially high value because stays in which the event has already occurred are also counted as events. We do so because urine-defined AKI may already have occurred; however, creatinine-defined AKI may still occur, and we still want to model it. Additionally, one can exit the event at the next time point. During training, the loss function is set to zero for events that have already occurred, since this can be inferred from the observation at that moment. The median remaining ICU-exit time is $54.7$ hours.

\paragraph{Covariates and preprocessing.}
Feature construction follows the ICU preprocessing pipeline of \citet{jonkers_beyond_2026}, adapted to the covariates used in this study. Each prediction row has $22$ standardized covariates, including gender, age at admission, weight, chronic renal failure status, baseline serum creatinine, elapsed ICU time, and $72$-hour summary features for serum creatinine and urine output: count, mean, minimum, maximum, latest value, and slope. We additionally include missingness and slope-availability indicators for the two time-varying measurements. Imputation and scaling are fit on the training split only: missing counts are set to zero, missing measurement summaries are filled with training medians, missing slopes are set to zero, and all model features are standardized using training means and standard deviations.

\paragraph{Training objectives and baselines.}
The censored engression model is a noise-conditioned MLP generator with noise dimension $500$, hidden width $500$, depth $2$, and a softplus output transform. It is trained with AdamW, learning rate $2\cdot 10^{-4}$, weight decay $10^{-5}$, batch size $1024$, gradient clipping at $1$, and early stopping on the validation localized censored energy score. Each training update uses $16$ generated samples. A stay-aware minibatch sampler prevents two rows from the same ICU stay from appearing in the same minibatch. The naive engression baseline uses the same generator class but replaces the localized censored objective by an ordinary energy-score objective that treats the observed censored vector $Y$ as if it were the true event-time vector.

We compare against five likelihood-based baselines trained on the same processed tensors: independent Weibull, conditional MLP independent Weibull, Clayton copula Weibull, conditional MLP Clayton copula Weibull, and a joint discrete-grid likelihood model with $32$ bins. The conditional Weibull and conditional Clayton baselines use hidden width $500$ and depth $2$; the joint grid baseline is included since this AKI task is bivariate.

\paragraph{Evaluation.}
Since the ICU-exit censoring time is observed for every prediction row, we evaluate using localized scores. For the bivariate energy score, generated samples are first localized by
\begin{equation*}
    \psi^\flat_C(z_1,z_2)=(\min(z_1,C),\min(z_2,C)),
\end{equation*}
and are then compared with the observed censored vector $Y$. We report the joint localized energy score for rows with at least one endpoint still at risk, and the marginal localized energy scores after excluding rows in which the corresponding endpoint is already present at prediction time. The test set contributes $802{,}657$ rows to the joint score, $725{,}596$ rows to the creatinine marginal score, and $699{,}383$ rows to the urine-output marginal score.

We also compute localized Brier scores at fixed horizons for the composite endpoint ``either AKI''. For a horizon $\tau$, the censored predicted risk is set to one when $C\le \tau$, and otherwise uses the latent sampled risk, e.g.
\begin{equation*}
    F^{\flat}_{C,\tau}(x)=
    \begin{cases}
    1, & C\le \tau,\\
    \Pr_F(T_1\le \tau \text{ or } T_2\le \tau\mid X=x), & C>\tau.
    \end{cases}
\end{equation*}
The observed censored binary outcome is $\mathbb 1\{Y_1\le\tau \text{ or } Y_2\le\tau\}$. Top-$5\%$ and top-$10\%$ positive predictive value and sensitivity are reported only as secondary ranking summaries on the horizon-specific evaluable subset. They are not valid scores and should not be interpreted as estimates of deployment utility.

\begin{table*}[t]
\centering
\scriptsize
\setlength{\tabcolsep}{3.5pt}
\renewcommand{\arraystretch}{1.05}
\caption{
Secondary utility metrics for the composite endpoint (either AKI). 
Localized Brier is a score, so lower is better; PPV and sensitivity are utility metrics, so higher is better. 
These metrics are descriptive only. 
Bold entries indicate the best method within each horizon and metric.
}
\label{tab:aki-utility-compact}
\begin{adjustbox}{max width=\textwidth}
\begin{tabular}{llccccc}
\toprule
Horizon & Method 
& Loc. Brier $\downarrow$ 
& PPV@5\% $\uparrow$ 
& Sens.@5\% $\uparrow$ 
& PPV@10\% $\uparrow$ 
& Sens.@10\% $\uparrow$ \\
\midrule

6h  & Naive                  & 0.023 & 0.485 & 0.528 & 0.338 & 0.735 \\
6h  & Censored               & \textbf{0.022} & \textbf{0.501} & \textbf{0.545} & \textbf{0.348} & \textbf{0.757} \\
6h  & Joint grid             & 0.028 & 0.465 & 0.505 & 0.341 & 0.742 \\
6h  & Indep. Weibull         & 0.223 & 0.177 & 0.193 & 0.157 & 0.342 \\
6h  & Clayton Weibull        & 0.206 & 0.179 & 0.194 & 0.159 & 0.345 \\
6h  & Cond. indep. Weibull   & 0.041 & 0.403 & 0.439 & 0.313 & 0.680 \\
6h  & Cond. Clayton Weibull  & 0.047 & 0.366 & 0.398 & 0.290 & 0.632 \\

\midrule
12h & Naive                  & 0.045 & 0.727 & 0.316 & 0.615 & 0.535 \\
12h & Censored               & \textbf{0.040} & \textbf{0.782} & \textbf{0.340} & \textbf{0.664} & \textbf{0.578} \\
12h & Joint grid             & 0.043 & 0.728 & 0.317 & 0.622 & 0.542 \\
12h & Indep. Weibull         & 0.216 & 0.367 & 0.160 & 0.327 & 0.284 \\
12h & Clayton Weibull        & 0.201 & 0.369 & 0.160 & 0.330 & 0.287 \\
12h & Cond. indep. Weibull   & 0.051 & 0.655 & 0.285 & 0.560 & 0.488 \\
12h & Cond. Clayton Weibull  & 0.054 & 0.633 & 0.275 & 0.547 & 0.476 \\

\midrule
24h & Naive                  & 0.081 & 0.759 & 0.167 & 0.696 & 0.307 \\
24h & Censored               & \textbf{0.071} & \textbf{0.840} & \textbf{0.185} & \textbf{0.751} & \textbf{0.331} \\
24h & Joint grid             & 0.074 & 0.805 & 0.177 & 0.724 & 0.319 \\
24h & Indep. Weibull         & 0.190 & 0.514 & 0.113 & 0.470 & 0.207 \\
24h & Clayton Weibull        & 0.179 & 0.515 & 0.114 & 0.473 & 0.208 \\
24h & Cond. indep. Weibull   & 0.077 & 0.762 & 0.168 & 0.690 & 0.304 \\
24h & Cond. Clayton Weibull  & 0.078 & 0.757 & 0.167 & 0.689 & 0.304 \\

\midrule
48h & Naive                  & 0.095 & 0.769 & 0.094 & 0.763 & 0.186 \\
48h & Censored               & \textbf{0.080} & \textbf{0.929} & \textbf{0.113} & \textbf{0.882} & \textbf{0.215} \\
48h & Joint grid             & 0.081 & 0.911 & 0.111 & 0.862 & 0.211 \\
48h & Indep. Weibull         & 0.138 & 0.676 & 0.083 & 0.648 & 0.158 \\
48h & Clayton Weibull        & 0.132 & 0.678 & 0.083 & 0.649 & 0.158 \\
48h & Cond. indep. Weibull   & 0.083 & 0.878 & 0.107 & 0.831 & 0.203 \\
48h & Cond. Clayton Weibull  & 0.083 & 0.886 & 0.108 & 0.840 & 0.205 \\

\bottomrule
\end{tabular}
\end{adjustbox}
\end{table*}

\paragraph{Fixed-horizon results.}
Table~\ref{tab:aki-utility-compact} reports the composite endpoint at $6$, $12$, $24$, and $48$ hours. Censored engression obtains the best localized Brier score at all four horizons. The gap is smallest at $6$ hours and becomes more pronounced from $12$ to $48$ hours, where censoring by ICU exit is more consequential. The secondary enrichment metrics are directionally consistent with the score results: censored engression usually selects a higher-risk top tail than the naive generator and the likelihood baselines. These summaries are useful for clinical interpretation, but the localized Brier score is the primary fixed-horizon criterion.

\begin{table}[t]
\centering
\scriptsize
\setlength{\tabcolsep}{4pt}
\caption{Additional AKI use-case benchmark results using localized energy scores. Lower is better. These scores evaluate overall fit of the predicted event-time distributions.}
\label{tab:aki-es-baselines}
\begin{tabular}{lccc}
\toprule
Method & Joint localized ES & Creatinine localized ES & Urine localized ES \\
\midrule
Naive observed engression & 43.442 & 23.591 & 32.981 \\
Censored engression & \textbf{30.933} & \textbf{15.439} & \textbf{23.769} \\
Independent Weibull & 45.948 & 25.915 & 34.154 \\
Conditional MLP independent Weibull & 35.728 & 19.534 & 26.061 \\
Clayton copula + Weibull & 45.749 & 25.915 & 34.154 \\
Conditional MLP Clayton copula + Weibull & 34.006 & 18.863 & 24.606 \\
Joint discrete-grid likelihood & 31.675 & 16.210 & 24.057 \\
\bottomrule
\end{tabular}
\end{table}

\paragraph{Distributional results.}
Table~\ref{tab:aki-es-baselines} compares methods that can generate full bivariate event-time samples. Censored engression has the lowest joint localized energy score and the lowest marginal localized energy scores, which coincide with localized CRPS in the univariate \(\beta=1\) case, for both AKI definitions. The joint discrete-grid likelihood baseline is the closest competitor, which suggests that flexible joint modeling is important for this task; however, it remains slightly worse than censored engression and is tied to the low-dimensional bivariate setting. Conditional Weibull and conditional Clayton baselines substantially improve upon their non-conditional counterparts, confirming the importance of covariate-dependent event-time distributions, but they do not match the sample-based censored engression objective. Overall, the AKI use case supports the same methodological conclusion as the synthetic experiments: treating censored observations as fully observed event times can distort the learned distribution, whereas localizing the score at the observed censoring time yields better predictive distributions based on the observed data.

\paragraph{Interpretational caveats.}
The row-level event rates above are prediction-row rates, not patient-level incidence estimates, because the same ICU stay can contribute multiple prediction times. In addition, ICU discharge is an observed exit process rather than a randomized censoring mechanism. The reported localized scores, therefore, evaluate the predictive distribution under the stated observed-data censoring convention; they should not be interpreted as causal estimates of latent AKI incidence after discharge. The use case is intended to demonstrate the proposed scoring and training methodology, not to define a final clinical deployment model.

\end{document}